%% file: main.tex
\documentclass[sigconf]{acmart}
\AtBeginDocument{%
  }

\setcopyright{acmlicensed}
\copyrightyear{2018}
\acmYear{2018}
\acmDOI{XXXXXXX.XXXXXXX}
\acmConference[Conference acronym 'XX]{Make sure to enter the correct
  conference title from your rights confirmation email}{June 03--05,
  2018}{Woodstock, NY}
\acmISBN{978-1-4503-XXXX-X/2018/06}
\usepackage{enumitem}
\usepackage{subcaption}
\usepackage{wrapfig}
\usepackage{algorithm}
\usepackage{algorithmic}

\begin{document}

\title{\method{}: Universal Structural Basis Distillation for Source-Free Graph Domain Adaptation}

\author{Yingxu Wang}
\authornote{Both authors contributed equally to this research.}
\affiliation{%
  \institution{Mohamed bin Zayed University of Artificial Intelligence}
  \country{}
}
\email{yingxv.wang@gmail.com}

\author{Kunyu Zhang}
\authornotemark[1]
\affiliation{%
  \institution{Zhengzhou University}
  \country{}
}
\email{kunyu.zky@gmail.com}

\author{Mengzhu Wang}
\affiliation{%
  \institution{The Hebei University of Technology}
  \country{}
}
\email{dreamkily@gmail.com}

\author{Siyang Gao}
\affiliation{%
  \institution{The City University of Hong Kong}
  \country{}
}
\email{siyangao@cityu.edu.hk}

\author{Nan Yin}
\affiliation{%
  \institution{The City University of Hong Kong}
  \country{}
}
\email{yinnan8911@gmail.com}

\renewcommand{\shortauthors}{Trovato et al.}

\begin{abstract}
Source-Free Graph Domain Adaptation (SF-GDA) is pivotal for privacy-preserving knowledge transfer across graph datasets. Although recent works incorporate structural information, they implicitly condition adaptation on the smoothness priors of source-trained GNNs, thereby limiting their generalization to structurally distinct targets. This dependency becomes a critical bottleneck under significant topological shifts, where the source model misinterprets distinct topological patterns unseen in the source domain as noise, rendering pseudo-label-based adaptation unreliable.
To overcome this limitation, we propose the \textbf{U}niversal \textbf{S}tructural \textbf{B}asis \textbf{D}istillation (\method{}), a framework that shifts the paradigm from adapting a biased model to learning a universal structural basis for SF-GDA. 
Instead of adapting a biased source model to a specific target, our core idea is to construct a structure-agnostic basis that proactively covers the full spectrum of potential topological patterns. Specifically, \method{} employs a bi-level optimization framework to distill the source dataset into a compact structural basis. By enforcing the prototypes to span the full Dirichlet energy spectrum, the learned basis explicitly captures diverse topological motifs, ranging from low-frequency clusters to high-frequency chains, beyond those present in the source. This ensures that the learned basis creates a comprehensive structural covering capable of handling targets with disparate structures. For inference, we introduce a spectral-aware ensemble mechanism that dynamically activates the optimal prototype combination based on the spectral fingerprint of the target graph. Extensive experiments on benchmarks demonstrate that \method{} significantly outperforms state-of-the-art methods, particularly in scenarios with severe structural shifts, while achieving superior computational efficiency by decoupling the adaptation cost from the target data scale. 

\end{abstract}

\ccsdesc[500]{Mathematics of computing~Graph algorithms}
\ccsdesc[100]{Computing methodologies~Neural networks}

\keywords{}

\received{20 February 2007}
\received[revised]{12 March 2009}
\received[accepted]{5 June 2009}

\def\method{USBD}

\maketitle

\input{code/1_intro}
\input{code/2_related_work}

\input{code/4_method}
\input{code/5_experiments}

\input{code/6_conclusion}

\bibliographystyle{plain}
\bibliography{reference}

\appendix
\input{code/7_appendix}
\end{document}

%% file: code/1_intro.tex
\section{Introduction}

Source-Free Graph Domain Adaptation (SF-GDA) has emerged as a critical technique for transferring knowledge across graph datasets while protecting data privacy \citep{liang2020we, mao2024source}. Unlike traditional methods that require accessing both source and target data, SF-GDA operates in a stricter setting where the original source data is unavailable during adaptation. 
The mainstream SF-GDA paradigm adopts a decoupled two-stage framework \cite{zhang2024collaborate, mao2024source}. Initially, a model is optimized on the source domain to encapsulate transferable knowledge; subsequently, this source-trained model serves as an initialization for the target domain, where it is fine-tuned exclusively using unlabeled target data.
This setting is particularly valuable in privacy-sensitive applications, such as medical diagnosis (e.g., analyzing patient networks across hospitals) \cite{mao2024source} and molecular research \cite{zhang2024collaborate,jin2022empowering,luo2024rank}, where sharing raw data is often restricted by legal regulations or commercial interests.

To address the absence of source data, existing SF-GDA approaches generally follow two paradigms: feature-centric alignment and pseudo-label self-training. The former aligns feature distributions through information maximization or clustering objectives \cite{yang2021exploiting,yang2022attracting}, aiming to cluster target nodes near source-learned decision boundaries. Conversely, the latter incorporates structural priors to refine noisy predictions, typically by enforcing structural consistency \cite{mao2024source} or employing spectral ranking \cite{luo2024rank}. Despite these efforts, a fundamental limitation persists across both paradigms because they rely heavily on the inductive biases of the source model. Operating as low-pass filters \cite{wu2019simplifying}, these models implicitly assume that target graphs share similar spectral characteristics with the source, such as smooth signal transitions across the topology \cite{zhu2021shift}. Consequently, when the target domain exhibits severe topological distribution shifts, the source model’s spectral filters fail to capture the high-frequency patterns and treat them as noise. This misalignment renders pseudo-labels unreliable, leading to catastrophic error accumulation during the adaptation process.

While existing methods have made strides in feature alignment, adapting to disparate structures without source data remains a formidable task. Prior works fail to surmount three fundamental challenges:
\textit{\textbf{(1) Unobservability of Structural Shifts}}. Due to the unavailability of source data, measuring explicit structural distance between domains is infeasible \cite{mao2024source}. Existing methods rely on source model parameters as a proxy for source knowledge \cite{yang2021exploiting}. However, these parameters only encapsulate the spectral support of the source domain \cite{wu2019simplifying} and remain incapable of representing out-of-distribution structural patterns required by the target \cite{zhu2021shift, wu2022handling}.
\textit{\textbf{(2) Self-Reinforcing Spectral Bias}}. GNNs fundamentally function as graph spectral filters \cite{wu2019simplifying}. Since the source model is typically initialized as a low-pass filter, it cannot teach itself to recognize high-frequency structures through self-training because it systematically suppresses the signals needed for adaptation \cite{nt2019revisiting, bo2021beyond}. Consequently, prior works are trapped in a feedback loop and cannot break free from the inductive bias of the source \cite{zhu2020beyond}.
\textit{\textbf{(3) Trade-off between Efficiency and Universality}}. A robust SF-GDA solution should possess universality to handle arbitrary target structures \cite{wu2022handling}. However, current target-specific approaches necessitate computationally expensive retraining or iterative optimization for every new deployment \cite{jin2022empowering, jin2020graph}. Achieving universality without prohibitive computational costs remains a significant open problem.

\begin{figure}
\centering
\includegraphics[scale=0.2]{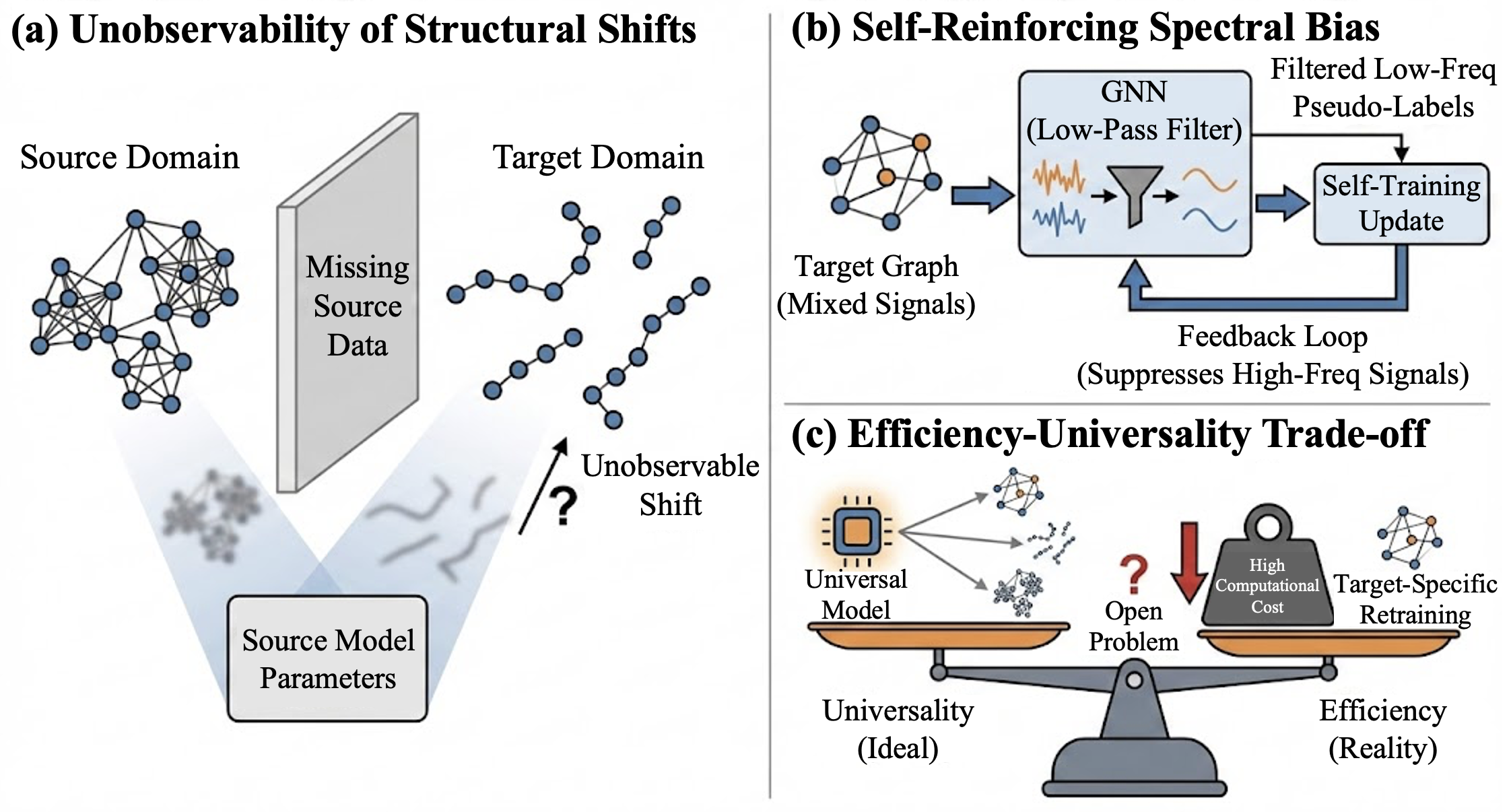}
\caption{The key challenges in SF-GDA. (a) Unobservability of structural shifts arises from missing source data, where source model fails to represent OOD target patterns. (b) The spectral bias dilemma: GNN low-pass filters remove necessary high-frequency signals, leading to a self-reinforcing error loop. (c) The efficiency-universality trade-off highlights the difficulty of handling arbitrary target structures without high computational costs during deployment.}
\vspace{-0.5cm}
\label{fig:challenge}
\end{figure}

In this paper, we propose a paradigm shift from relying on specific targets to building a universal covering model. To achieve this, we introduce a framework: Universal Structural Basis Distillation (\method{}) for source-free graph domain adaptation. 
Rather than adapting a biased model to structurally disparate targets, \method{} distills source knowledge into a compact set of synthetic prototype graphs that serve as a structural basis. 
Specifically, we address structural unobservability and spectral bias through a bi-level optimization framework. This hierarchical design explicitly optimizes the spectrum-conditional generator to maximize structural diversity. We then enforce the synthesized prototypes to span the full Dirichlet energy spectrum, compelling the generator to produce varied topological motifs. These patterns range from low-frequency clusters to high-frequency chains, effectively constructing a "universal toolset" that covers potential target structures. 
To address the efficiency-universality trade-off, we introduce a spectral-aware ensemble mechanism for inference. This allows the model to dynamically re-weight and activate the optimal combination of basis prototypes based on the spectral fingerprint of target graphs, enabling instant adaptation to arbitrary unseen targets without retraining.
Extensive experiments demonstrate that \method{} outperforms state-of-the-art methods on benchmarks with severe structural shifts, confirming the efficacy of our universal basis perspective.

Our contributions can be summarized as follows:
\begin{itemize}[itemsep=2pt,topsep=0pt,parsep=0pt,leftmargin=*]
\item We formally identify three fundamental bottlenecks in source-free structural adaptation: structural unobservability, self-reinforcing spectral bias, and the efficiency-universality trade-off.
\item We propose \method{}, a bi-level optimization framework that shifts from passive adaptation to proactive universal covering. By driving a spectrum-conditional generator to synthesize a full-spectrum "universal toolset," the proposed \method{} effectively covers diverse topological motifs even in the absence of source data.
\item Extensive experiments demonstrate that \method{} achieves state-of-the-art performance across benchmarks. Crucially, it resolves the efficiency dilemma by handling structural disparities without requiring computationally expensive target-specific retraining.
\end{itemize}

%% file: code/2_related_work.tex
\section{Related Work}

\textbf{{Source-Free Graph Domain Adaptation (SF-GDA).}} SF-GDA addresses the practical requirement of adapting models to target graphs without access to source data, thereby preserving data privacy and reducing storage overhead~\cite{luo2023source, luo2024gala,yu2025samgpt}. Early efforts in this area primarily adapt general source-free domain adaptation techniques, such as information maximization and pseudo-label self-training, to graph data to better align with graph-structured learning scenarios~\cite{mao2024source, luo2024rank, qu2024lead}. Subsequent works incorporate graph-specific inductive biases by enforcing structural consistency or leveraging graph contrastive learning to better handle domain shifts in the absence of source supervision~\cite{wu2024graph, zhang2024collaborate,wang2024degree}. Despite these advances, existing methods largely adhere to a model-centric adaptation paradigm, where successful transfer critically hinges on the initial transferability of source-pretrained parameters~\cite{yang2021exploiting, mao2024source}. When domain shifts are severe, particularly in the presence of substantial structural discrepancies, the source model often fails to produce reliable initial representations, resulting in noisy pseudo-labels on target domain and progressive error accumulation during self-training~\cite{zhao2024multi,zhang2025aggregate}. To mitigate this limitation, we propose \method{}, a data-centric framework that reduces reliance on biased source-pretrained parameters by distilling source knowledge into a spectrum-spanning graph basis, thereby enabling robust and efficient generalization across structurally disparate target domains

\begin{figure*}[t]
\centering
\includegraphics[scale=0.24]{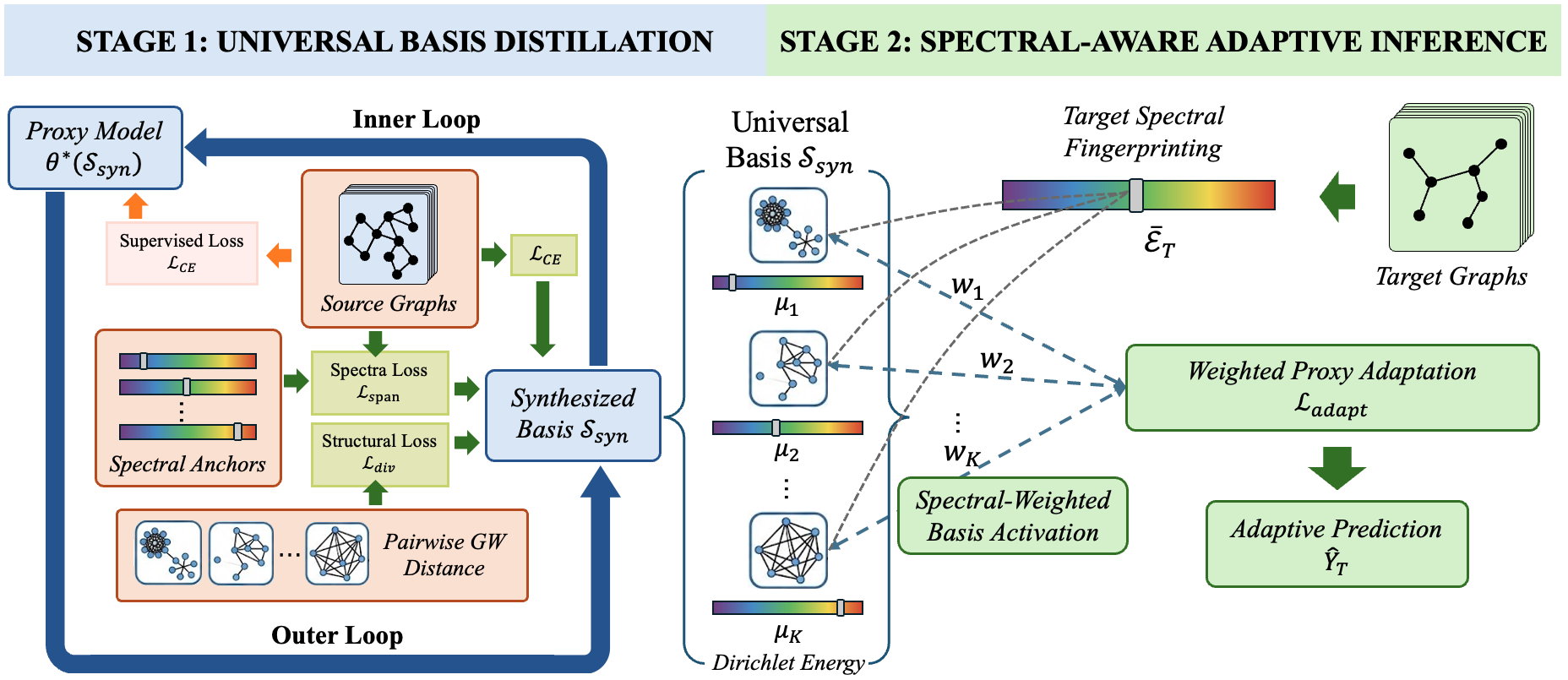}
\caption{Proposed USBD comprises two stages: (a) Universal Structural Basis Distillation: A bi-level optimization framework distills source knowledge into compact synthetic prototypes spanning the full Dirichlet energy spectrum to capture diverse topological motifs beyond the source domain. (b) Spectral-Aware Adaptive Inference: The module identifies the target spectral fingerprint and dynamically activates optimal basis combinations of the learned prototypes for instant adaptation.}
\label{fig:framework}
\end{figure*}

\textbf{{Out-of-Distribution (OOD) Generalization on Graphs.}}
OOD generalization on graphs focuses on improving the robustness of Graph Neural Networks (GNNs) to distribution shifts in graph-structured data, a property that is critical for real-world applications such as dynamic social networks, molecular discovery, and recommendation systems, where graph distributions evolve over time~\cite{li2022learning, yao2024empowering, Wang2025ProtoMolEM}. Most existing OOD generalization approaches focus on invariant learning, which aims to identify stable substructures that remain predictive across diverse environments~\cite{yao2025pruning, li2025out,Wang2025NestedGP}. While theoretically grounded, such methods often incur substantial computational overhead and struggle to model complex global structural shifts that extend beyond localized invariances~\cite{guo2024investigating, zhang2024survey}. More recently, a spectral perspective has emerged as an effective lens for characterizing graph distribution shifts, revealing that different domains often exhibit distinct spectral energy profiles~\cite{zhang2023spectral, gu2025spectralgap,liu2024rethinking}. However, existing methods typically rely on the limited spectral support of the source data during training and therefore fail to extrapolate to target domains that occupy disjoint regions of the spectral energy landscape~\cite{liu2023good, du2024and,wang2025bridging}. To overcome this limitation, we propose \method{}, a spectral extrapolation framework that transcends the restricted support of source data by synthesizing a universal graph basis explicitly constrained to span the full energy spectrum, thereby enabling robust generalization to unseen target domains with disjoint spectral profiles.

%% file: code/4_method.tex
\section{Methodology}

\subsection{Problem Formulation}
We consider the problem of Source-Free Graph Domain Adaptation (SF-GDA) for graph classification. Formally, let $\mathcal{D}_S = \{(\mathcal{G}_i^s, y_i^s)\}_{i=1}^{N_S}$ denote a labeled source domain and $\mathcal{D}_T = \{\mathcal{G}_j^t\}_{j=1}^{N_T}$ an unlabeled target domain, where each graph instance $\mathcal{G} = (\mathcal{V}, E, \mathbf{X})$ is associated with a graph-level label $y \in \mathcal{Y}$. A key challenge lies in the structural disparity between domains, where the underlying graph generative distributions differ, i.e., $P_S(\mathcal{G}) \neq P_T(\mathcal{G})$. Unlike standard adaptation settings, we operate under strict source-free constraints where $\mathcal{D}_S$ is inaccessible during adaptation; only a GNN model $f_S$ pre-trained on $\mathcal{D}_S$ is provided. Consequently, our objective is to adapt the fixed source model $f_S$ to the target domain using only the unlabeled graphs in $\mathcal{D}_T$, maximizing classification accuracy by effectively aligning the disparate structural distributions.

\subsection{Overview of \method{}}
In this paper, we propose a novel framework \method{} for SF-GDA, as shown in Figure.~\ref{fig:framework}. \method{} comprises two key components: 
(1) \textbf{Universal Structural Basis Distillation.} To address structural unobservability and spectral bias, this module employs a bi-level optimization framework to distill source knowledge into a dictionary of prototypes. By explicitly constraining the synthetic basis to cover the full Dirichlet energy interval, we provide explicit guidance for the optimization of synthetic tensors and ensure representation completeness for arbitrary topological patterns.
(2) \textbf{Spectral-Aware Adaptive Inference.} To resolve the efficiency-universality trade-off, we introduce a spectral-aware ensemble mechanism that dynamically calibrates prototypes based on the spectral fingerprint of target graphs. This strategy minimizes the structural covering discrepancy, which is a tight upper bound on the target risk, enabling instant and robust adaptation without expensive retraining.

\subsection{Universal Structural Basis Distillation}

The primary challenge in SF-GDA lies in the structural unobservability of the target domain \cite{liang2020we, fang2023sourcefreeunsuperviseddomainadaptation}, obscuring potential topological shifts and leading to significant distribution discrepancies \cite{zhu2020beyond, chien2021adaptive}. Existing generative methods typically address this by maximizing source likelihood to synthesize prototypes \cite{kipf2016variational, you2018graphrnn}. However, this paradigm suffers from \textit{spectral bias} \cite{rahaman2019spectral}. Since source domains often exhibit a limited frequency range, optimization tends to overfit this specific profile. Consequently, spectral analysis \cite{nt2019revisiting} reveals these models function inherently as low-pass filters, constraining their ability to synthesize diverse structural motifs necessary for Out-of-Distribution (OOD) topological patterns. To transcend this, we propose {Universal Structural Basis Distillation}, a bi-level learning framework decoupling basis construction from source density. By introducing explicit Dirichlet energy anchors, we compel the optimization to extrapolate beyond source support and synthesize orthogonal prototypes spanning the full spectral interval, constructing a complete structural basis robust to arbitrary shifts.

\subsubsection{{Universal Basis Condensation Formulation.}}
We aim to synthesize a compact universal basis $\mathcal{S}_{syn} = \{(A_k, X_k, Y_k)\}_{k=1}^K$ that encapsulates the structural diversity of the source domain while spanning the full spectral spectrum. In formulation, $\mathcal{S}_{syn}$ denotes a learnable structural manifold, explicitly optimized to cover the topological gap between low- and high-frequency patterns. To enforce spectral coverage without compromising semantic fidelity, we formulate the basis construction as a bi-level learning problem. Let $\theta$ denote the parameters of a proxy GNN classifier. We define the optimization objective as the following hierarchical problem:
\begin{subequations}
\begin{align}
    \min_{\mathcal{S}_{syn}} \mathcal{L}_{\text{meta}} &={\mathcal{L}_{\text{sem}}(\mathcal{S}_{syn})} + \lambda_1\mathcal{L}_{\text{span}}(\mathcal{S}_{syn}) + \lambda_2\mathcal{L}_{\text{div}}(\mathcal{S}_{syn}), \label{eq:outer_meta} \\
    \text{s.t. } \theta^*(&\mathcal{S}_{syn}) = \operatorname*{arg\,min}_{\theta} \mathcal{L}_{\text{CE}}(\text{GNN}(\mathcal{S}_{syn}; \theta), Y_{syn}). \label{eq:inner_proxy}
\end{align}
\end{subequations}
Here, $\mathcal{L}_{\text{CE}}$ denotes the standard Cross-Entropy loss function. Eq.~(\ref{eq:inner_proxy}) constrains the proxy model $\theta^*$ to be the optimal solution on the synthetic basis, establishing a link between the basis structure and model performance. Eq.~(\ref{eq:outer_meta}) drives the basis update, balancing the preservation of source semantics with the maximization of spectral and structural coverage.

\paragraph{\textbf{Semantic Fidelity via Meta-Matching.}}
The term $\mathcal{L}_{\text{sem}}$ serves as a semantic anchor, ensuring that any structural variation introduced by the condensation process does not degrade class discriminability. Since the target domain is unobserved, we leverage the real source data $\mathcal{D}_S$ as the ground truth. The semantic loss is computed by evaluating the proxy model $\theta^*$ (trained on $\mathcal{S}_{syn}$) on the real source distribution:
\begin{equation}
    \mathcal{L}_{\text{sem}}(\mathcal{S}_{syn}) = \mathbb{E}_{(G, y) \sim \mathcal{D}_S} \left[ \mathcal{L}_{\text{CE}}(\text{GNN}(G; \theta^*(\mathcal{S}_{syn})), y) \right].
\end{equation}
This formulation treats the synthetic basis as hyper-parameters of the learning process. By differentiating through the inner loop Eq.~(\ref{eq:inner_proxy}), we force the synthetic graphs to distill the essential decision boundaries of the source domain, ensuring that they remain semantically valid prototypes rather than random noise.

\paragraph{\textbf{Spectral and Structural Regularization.}}
To strictly align the basis with the target's frequency profiles, we introduce a set of spectral anchors $\mathcal{M} = \{\mu_1, \dots, \mu_K\}$ that linearly span the Dirichlet energy interval $[0, E_{\max}]$. We impose a hard constraint on the generated topology, forcing the Dirichlet energy $\mathcal{E}(G_k)$ of each basis graph $G_k$ to match its corresponding topological coordinate $\mu_k$:
\begin{equation}
    \mathcal{L}_{\text{span}} = \sum_{k=1}^{K} \left\| \frac{\operatorname{Tr}(X_k^\top \hat{L}_k X_k)}{\|X_k\|_F^2} - \mu_k \right\|^2,
\end{equation}
where $\hat{\mathbf{L}}_k$ denotes the normalized Laplacian of $A_k$. This term explicitly steers the condensation process. Even when the source domain $\mathcal{D}_S$ is dominated by low-frequency signals, $\mathcal{L}_{\text{span}}$ enforces the synthesis of high-energy prototypes while preserving semantic consistency, thereby mitigating source spectral bias.

However, spectral diversity alone does not ensure structural uniqueness. While $\mathcal{L}_{\text{span}}$ enforces coverage across the frequency domain, it fails to penalize topological redundancy. To mitigate the risk of mode collapse, where multiple basis graphs degenerate into identical geometries, we complementarily maximize the pairwise Gromov-Wasserstein (GW) distance between basis elements:
\begin{equation}
    \mathcal{L}_{\text{div}} = - \sum_{i} \sum_{j > i} GW(A_i, A_j).
\end{equation}
By maximizing the pairwise Gromov-Wasserstein discrepancy, we explicitly enforce topological distinctiveness among basis elements. This maximization of structural variance prevents redundancy and effectively expands the {support} of the synthesized basis, ensuring comprehensive coverage of the potential target distribution.

\subsubsection{{Theoretical Guarantee of Universality.}}
The proposed bi-level optimization explicitly constructs a discrete set of structural prototypes. However, a fundamental theoretical question remains: \textit{Can a finite basis set $\mathcal{S}_{syn}$ sufficiently cover the continuous and potentially infinite space of target structural shifts?} Since the target domain is unobserved during training, we must ensure that for any arbitrary target graph, regardless of its position on the spectral energy spectrum, there exists a close match within our basis to guarantee effective adaptation. To address this, we establish the following property, which provides a theoretical error bound on the spectral coverage of our distilled basis.

\begin{theorem}[Universal Spectral Covering Property]
\label{thm:coverage}
Let $\mathcal{H}_{\text{spec}} = [0, E_{\max}]$ be the bounded spectral domain of valid Dirichlet energies. Assume the synthesized basis $\mathcal{S}_{syn} = \{G_k\}_{k=1}^K$ is optimized such that the spectral coverage loss $\mathcal{L}_{\text{span}}$ converges to zero. Let the anchors $\{\mu_k\}_{k=1}^K$ be uniformly distributed over $\mathcal{H}_{\text{spec}}$ with a grid spacing of $\delta = \frac{E_{\max}}{K-1}$.
For any unseen target graph $G_T$ with Dirichlet energy $\mathcal{E}(G_T) \in \mathcal{H}_{\text{spec}}$, there exists a basis prototype $G_{k^*} \in \mathcal{S}_{syn}$ such that the spectral approximation error is bounded by:
\begin{equation}
    \min_{G_k \in \mathcal{S}_{syn}} | \mathcal{E}(G_T) - \mathcal{E}(G_k) | \le \frac{\delta}{2} + \epsilon,
\end{equation}
where $\epsilon$ is the optimization residual.
\end{theorem}
Theorem~\ref{thm:coverage} provides a rigorous theoretical certification for the universality of our framework. It implies that as the basis size $K$ increases (reducing $\delta$) and optimization converges (minimizing $\epsilon$), the synthesized set $\mathcal{S}_{syn}$ asymptotically forms a dense $\delta$-cover of the entire spectral manifold. Consequently, this result circumvents the fundamental hurdle of target structural unobservability: it guarantees that any unseen target, regardless of its position on the spectral energy spectrum, falls within the reliable support of our pre-computed basis. This theoretical completeness legitimizes our inference paradigm, ensuring that the adaptation phase is reduced to a robust matching problem rather than a blind search.

\subsection{Spectral-Aware Adaptive Inference}

With the Universal Structural Basis $\mathcal{S}_{syn}$ constructed, the critical challenge during inference shifts to balancing adaptation universality with computational efficiency. Existing SF-GDA paradigms typically struggle to reconcile these conflicting objectives. On one hand, static models or lightweight projections \cite{ijcai2021p402} offer high inference speed but lack the structural flexibility to handle drastic topological shifts across the spectral energy spectrum, limiting their universality. On the other hand, extensive Test-Time Adaptation (TTA) methods \cite{liu2022confidence} can improve coverage via iterative self-training, but they incur prohibitive computational costs and latency, rendering them impractical for real-time deployment. This creates a fundamental \textit{efficiency-universality challenge}: how to achieve robust adaptation across arbitrary target structures without conducting expensive gradient updates on high-dimensional target data?

To resolve this challenge, we propose {Spectral-Aware Adaptive Inference}, a framework that harmonizes universality with efficiency by reformulating adaptation as a {Weighted Proxy Learning} task. Rather than conducting expensive iterative optimization on the target domain, we utilize the pre-computed Universal Basis $\mathcal{S}_{syn}$ as a compact, spectrally-indexed dictionary. Upon extracting the target's spectral fingerprint, we derive a sparse weight vector to dynamically activate the relevant basis prototypes. This effectively transforms the intractable adaptation problem into a rapid, lightweight calibration on a condensed synthetic set. Our strategy thereby guarantees {universality} through the basis's dense spectral coverage, while maximizing {efficiency} by obviating the need for recurrent back-propagation on target graphs.

\subsubsection{{Target Spectral Fingerprinting.}}
To profile the structural shift of the unobserved target domain $\mathcal{D}_T$, we compute its global spectral signature. We adopt the \textit{Normalized Dirichlet Energy} as a scale-invariant metric to quantify the operating point of the target distribution. Specifically, we aggregate the energy of all graphs $G_j = (\mathbf{A}_j, \mathbf{X}_j) \in \mathcal{D}_T$ to obtain the domain-level fingerprint $\bar{\mathcal{E}}_T$:
\begin{equation}
\label{Fingerprinting}
    \bar{\mathcal{E}}_T = \frac{1}{|\mathcal{D}_T|} \sum_{G_j \in \mathcal{D}_T} \frac{\operatorname{Tr}(\mathbf{X}_j^\top \hat{\mathbf{L}}_j \mathbf{X}_j)}{\|\mathbf{X}_j\|_F^2},
\end{equation}
where $\hat{\mathbf{L}}_j$ denotes the normalized Laplacian of $A_j$. This scalar $\bar{\mathcal{E}}_T$ serves as a robust topological anchor that effectively summarizes the global bias of the target domain, independent of semantic labels.

\subsubsection{{Spectral-Weighted Basis Activation.}}
To strictly align the inference model with the target's specific spectral profile, we employ a \textit{soft activation} mechanism over the Universal Structural Basis $\mathcal{S}_{syn} = \{(A_k, X_k, Y_k)\}_{k=1}^K$. Since the spectral manifold is continuous, the optimal source knowledge for a specific $\bar{\mathcal{E}}_T$ typically corresponds to a weighted combination of discrete basis prototypes rather than a single element.
We utilize a Kernel-based Spectral Attention to compute the importance weights $\mathbf{w} \in \mathbb{R}^K$. For each basis graph, the activation weight $w_k$ is derived from the proximity between its anchor energy $\mu_k$ and the target fingerprint $\bar{\mathcal{E}}_T$:
\begin{equation}
    w_k = \frac{\exp\left( - \frac{\| \bar{\mathcal{E}}_T - \mu_k \|^2}{2\sigma^2} \right)}{\sum_{j=1}^K \exp\left( - \frac{\| \bar{\mathcal{E}}_T - \mu_j \|^2}{2\sigma^2} \right)},
    \label{eq:attention_weights}
\end{equation}
where $\sigma$ acts as a temperature parameter controlling the bandwidth of the spectral filter. A smaller $\sigma$ enforces strict nearest-neighbor matching, while a larger $\sigma$ encourages broader basis participation.

Guided by these adaptive weights, rather than physically synthesizing a composite graph, we formulate the adaptation as a Weighted Proxy Learning task. Specifically, we instantiate a target-specific classifier $\phi$ and optimize it by minimizing the importance-weighted cross-entropy loss over the synthetic basis:
\begin{equation}
    \min_\phi \mathcal{L}_{\text{adapt}}(\phi) = \sum_{k=1}^K w_k \cdot \mathcal{L}_{\text{CE}}(\text{GNN}(A_k, X_k; \phi), Y_k).
    \label{eq:weighted_proxy}
\end{equation}
This objective effectively re-weights the source knowledge distribution. It forces the proxy model to prioritize decision boundaries derived from prototypes that are spectrally isomorphic to the target domain, thereby achieving accurate adaptation without exposing the model to noisy pseudo-labels from the unobserved target data.

\subsubsection{{Adaptive Prediction.}}
Finally, prediction is performed by feeding the original target graph $\mathcal{G}_T = (\mathbf{A}_T, \mathbf{X}_T)$ directly into the adapted proxy classifier $\phi^*$. Unlike traditional methods that require modifying the input structure to fit a frozen model, our approach adapts the model decision boundary to the target's geometry.
\begin{equation}
    \hat{\mathbf{Y}}_T = \text{GNN}(\mathbf{A}_T, \mathbf{X}_T; \phi^*),
\end{equation}
where $\phi^*$ is the optimal solution from Eq.~(\ref{eq:weighted_proxy}). Since $\phi^*$ is trained on the synthesized basis that strictly matches the target's spectral fingerprint $\bar{\mathcal{E}}_T$, it inherently generalizes to the unseen target distribution. This inference paradigm bypasses the risk of "structure-mismatch" error, ensuring that the classification logic is fully compatible with the target's specific spectral energy profile.

Crucially, this architecture resolves the efficiency-universality trade-off through two mechanisms. Regarding \textbf{Efficiency}, the adaptation cost is decoupled from the target data scale. Since the optimization of $\phi^*$ (Eq.~\ref{eq:weighted_proxy}) is restricted to the compact basis $\mathcal{S}_{syn}$ (where $|\mathcal{S}_{syn}| \ll |\mathcal{D}_T|$), the computational overhead is negligible compared to standard Test-Time Training (TTT) that requires back-propagation on high-dimensional target graphs. Regarding \textbf{Universality}, we differ fundamentally from approaches that attempt to heuristically {modify the input structure} to fit fixed smoothness assumptions. Such structural modification is often rigid and computationally expensive. In contrast, our approach {recalibrates the model's decision boundary} to accommodate the intrinsic geometry of the target. This mechanism allows the classifier to flexibly align with arbitrary spectral patterns across the full Dirichlet energy spectrum, thereby ensuring universal coverage without the artifacts introduced by forced structural changes.

\subsubsection{{Generalization Bound Analysis.}}
To substantiate the validity of our inference paradigm, we analyze the generalization bound on the target domain. Our objective is to prove that minimizing the spectral discrepancy between the target graph and the synthesized basis directly translates to a tighter upper bound on the prediction risk. The following theorem formalizes this dependency, establishing that the generalization capability of \method{} is strictly governed by the spectral covering density of the Universal Basis.

\begin{theorem}[Generalization Bound via Universal Covering]
\label{thm:generalization}
Let $f$ be the graph encoder and $h$ be the hypothesis (classifier). Let $\mathcal{R}_{\mathcal{T}}(h \circ f)$ denote the expected risk on the target domain, and $\hat{\mathcal{R}}_{\mathcal{S}_{syn}}(h \circ f)$ denote the empirical risk on the synthesized universal basis $\mathcal{S}_{syn}$.
Assume the loss function $\ell$ is $L_{\ell}$-Lipschitz continuous with respect to the graph representation, and the encoder $f$ is $L_{f}$-Lipschitz continuous with respect to the Dirichlet energy spectrum.
Then, with probability at least $1-\delta$, the target risk is bounded by:
\begin{equation}
    \mathcal{R}_{\mathcal{T}}(h \circ f) \leq \hat{\mathcal{R}}_{\mathcal{S}_{syn}}(h \circ f) + K_{\text{Lip}} \cdot \mathbb{E}_{G \sim \mathcal{D}_T} \left[ \min_{G_k \in \mathcal{S}_{syn}} | \mathcal{E}(G) - \mathcal{E}(G_k) | \right] + \lambda,\nonumber
\end{equation}
where $K_{\text{Lip}} = L_{\ell} L_{f}$ is the composite Lipschitz constant, and $\lambda$ represents the minimal combined risk of the optimal hypothesis.
\end{theorem}

Theorem \ref{thm:generalization} reveals that the target risk is bounded by the source risk and the {structural covering discrepancy}. Theorem \ref{thm:coverage} guarantees that this covering discrepancy is bounded by $\mathcal{O}(\delta + \epsilon)$. This implies that by minimizing the risk on our synthesized universal basis and ensuring dense spectral coverage, we tighten the upper bound on unseen target domains without requiring target labels.

\begin{algorithm}[h]
\caption{Complete Framework of \method{}}
\label{alg:overall}
\begin{algorithmic}[1]
\REQUIRE Source Data $\mathcal{D}_S$, Target Data $\mathcal{D}_T$ (unlabeled), Basis Size $K$, iterations $T$, Anchors $\{\mu_k\}_{k=1}^K$.
\ENSURE Target Prediction $\hat{\mathbf{Y}}_T$.

\STATE \textbf{--- Stage 1: Universal Basis Distillation ---}
\STATE Initialize synthetic basis $\mathcal{S}_{syn} = \{(A_k, X_k, Y_k)\}_{k=1}^K$ randomly.
\WHILE{not converged}
    \STATE \textit{// Inner Loop: Update Proxy Model}
    \STATE Calculate $\theta^*(\mathcal{S}_{syn})$ by minimizing $\mathcal{L}_{\text{CE}}$ on $\mathcal{S}_{syn}$ (Eq.~\eqref{eq:inner_proxy}).
    
    \STATE \textit{// Outer Loop: Update Universal Basis}
    \STATE Compute Semantic Loss $\mathcal{L}_{\text{sem}}$ on real source batch $\mathcal{D}_S$.
    \STATE Compute Spectral Loss $\mathcal{L}_{\text{span}}$ using anchors $\{\mu_k\}$ (Eq.~\eqref{eq:outer_meta}).
    \STATE Compute Diversity Loss $\mathcal{L}_{\text{div}}$ via GW distance.
    \STATE Update $\mathcal{S}_{syn} \leftarrow \mathcal{S}_{syn} - \eta \nabla_{\mathcal{S}_{syn}} (\mathcal{L}_{\text{sem}} + \lambda_1 \mathcal{L}_{\text{span}} + \lambda_2 \mathcal{L}_{\text{div}})$.
\ENDWHILE
\STATE \textbf{Output:} Optimized Universal Basis $\mathcal{S}_{syn}$.

\STATE \textbf{--- Stage 2: Spectral-Aware Adaptive Inference ---}
\STATE \textit{// Step 1: Target Fingerprinting}
\STATE Compute target domain spectral signature $\bar{\mathcal{E}}_T$ with Eq.~\eqref{Fingerprinting}.

\STATE \textit{// Step 2: Basis Activation}
\STATE Compute importance weights $\mathbf{w}$ based on $\|\bar{\mathcal{E}}_T - \mu_k\|$ (Eq.~\eqref{eq:attention_weights}).

\STATE \textit{// Step 3: Proxy Adaptation (Label-free)}
\STATE Initialize proxy classifier $\phi$.
\FOR{iter $= 1$ to $T$}
    \STATE Sample batch $(A_k, X_k, Y_k)$ from $\mathcal{S}_{syn}$.
    \STATE Update $\phi$ by minimizing weighted loss with Eq.~\eqref{eq:weighted_proxy}.
\ENDFOR

\STATE \textit{// Step 4: Prediction}
\RETURN Target Prediction $\hat{\mathbf{Y}}_T = \text{GNN}(\mathbf{A}_T, \mathbf{X}_T; \phi^*)$.
\end{algorithmic}
\end{algorithm}

\input{table/table_combination}

\subsection{Learning Framework}
\label{sec:learning_framework}

The protocol is detailed in Algorithm \ref{alg:overall}. Treating the basis $\mathcal{S}_{syn}$ as learnable tensors, we proceed with a bi-level optimization. In the \textit{inner loop} (Line 5), we obtain the optimal proxy parameters $\theta^*$ exclusively on the synthetic basis to establish the differentiable coupling between topology and decision boundaries. In the \textit{outer loop} (Lines 7-10), we validate the proxy on real source data ($\mathcal{L}_{\text{sem}}$) under spectral constraints ($\mathcal{L}_{\text{span}}, \mathcal{L}_{\text{div}}$), and update $\mathcal{S}_{syn}$ via meta-gradients.
During testing, we perform label-free adaptation. We first profile the target's spectral signature $\bar{\mathcal{E}}_T$ (Line 15) to compute basis activation weights $\mathbf{w}$ (Line 17). Instead of predicting with a fixed model, we train a lightweight target-specific proxy $\phi$ on the \textit{weighted} basis (Lines 20-23). Finally, predictions are generated by feeding the original target graph into this adapted proxy (Line 25), ensuring spectral alignment without modifying the input structure.

%% file: table/table_combination.tex
\begin{table*}[h]
\small
\centering
\caption{Graph classification results (in \%) under node and edge density domain shifts on the Mutagenicity dataset, and 
under correlation shifts (Corr. Shift) and feature shifts (source$\rightarrow$target). S, SB, P, D, C, CM, B, and BM denote Spurious-Motif, Spurious-Motif\_bias, PROTEINS, DD, COX2, COX2\_MD, BZR, and BZR\_MD, respectively. \textbf{Bold} indicates the best performance.}
\vspace{-0.1cm}
\resizebox{1.0\textwidth}{!}{
\begin{tabular}{l|c|c|c|c|c|c|c|c|c|c|c|c|c|c}
\toprule
\textbf{Methods}
& \multicolumn{3}{c|}{\textbf{Node Shift}}
& \multicolumn{3}{c|}{\textbf{Edge Shift}}
& \multicolumn{2}{c|}{\textbf{Corr. Shift}} 
& \multicolumn{6}{c}{\textbf{Feature Shift}} \\
\cmidrule(lr){2-4} \cmidrule(lr){5-7} \cmidrule(lr){8-15}
& M0$\rightarrow$M1 & M0$\rightarrow$M2 & M0$\rightarrow$M3
& M0$\rightarrow$M1 & M0$\rightarrow$M2 & M0$\rightarrow$M3
& S$\rightarrow$SB & SB$\rightarrow$S & P$\rightarrow$D & D$\rightarrow$P & C$\rightarrow$CM & CM$\rightarrow$C & B$\rightarrow$BM & BM$\rightarrow$B \\
\midrule

WL subtree
& 34.3 & 40.4 & 52.7
& 34.4 & 47.6 & 52.7
& 51.3 & 33.7 & 43.0 & 42.2 & 53.1 & 58.2 & 51.3 & 44.0 \\

GCN
& 64.1\scriptsize{$\pm$1.4} & 65.5\scriptsize{$\pm$2.0} & 56.9\scriptsize{$\pm$2.1}
& 66.3\scriptsize{$\pm$1.7} & 63.6\scriptsize{$\pm$1.4} & 56.0\scriptsize{$\pm$1.4}
& 33.8\scriptsize{$\pm$1.4} & 36.9\scriptsize{$\pm$1.0} & 48.9\scriptsize{$\pm$2.0} & 60.9\scriptsize{$\pm$2.3} & 51.0\scriptsize{$\pm$1.8} & 66.9\scriptsize{$\pm$1.8} & 48.7\scriptsize{$\pm$2.0} & 78.8\scriptsize{$\pm$1.7} \\

GIN
& 66.5\scriptsize{$\pm$2.1} & 52.0\scriptsize{$\pm$1.7} & 53.7\scriptsize{$\pm$1.7}
& 67.1\scriptsize{$\pm$1.7} & 54.2\scriptsize{$\pm$2.6} & 55.4\scriptsize{$\pm$1.9}
& 45.0\scriptsize{$\pm$2.4} & 48.2\scriptsize{$\pm$1.9} & 57.3\scriptsize{$\pm$2.2} & 61.9\scriptsize{$\pm$1.9} & 53.8\scriptsize{$\pm$2.5} & 55.6\scriptsize{$\pm$2.0} & 49.9\scriptsize{$\pm$2.4} & 79.2\scriptsize{$\pm$2.8} \\

GMT
& 65.7\scriptsize{$\pm$1.8} & 62.1\scriptsize{$\pm$2.1} & 59.0\scriptsize{$\pm$2.0}
& 67.9\scriptsize{$\pm$1.3} & 61.5\scriptsize{$\pm$1.8} & 58.2\scriptsize{$\pm$2.4}
& 35.7\scriptsize{$\pm$2.0} & 43.1\scriptsize{$\pm$2.0} & 59.5\scriptsize{$\pm$2.5} & 50.7\scriptsize{$\pm$2.2} & 49.3\scriptsize{$\pm$1.8} & 58.2\scriptsize{$\pm$2.0} & 50.2\scriptsize{$\pm$2.3} & 74.4\scriptsize{$\pm$1.8} \\

CIN
& 65.1\scriptsize{$\pm$1.7} & 66.0\scriptsize{$\pm$1.7} & 55.2\scriptsize{$\pm$1.5}
& 66.3\scriptsize{$\pm$1.8} & 60.8\scriptsize{$\pm$1.7} & 55.8\scriptsize{$\pm$2.4}
& 43.8\scriptsize{$\pm$2.2} & 47.4\scriptsize{$\pm$1.8} & 59.1\scriptsize{$\pm$2.6} & 58.0\scriptsize{$\pm$2.7} & 51.2\scriptsize{$\pm$2.0} & 55.6\scriptsize{$\pm$1.5} & 49.2\scriptsize{$\pm$1.4} & 74.2\scriptsize{$\pm$1.9} \\

PathNN
& 70.2\scriptsize{$\pm$1.5} & 67.1\scriptsize{$\pm$2.0} & 58.0\scriptsize{$\pm$1.9}
& 68.9\scriptsize{$\pm$1.9} & 62.9\scriptsize{$\pm$1.7} & 58.1\scriptsize{$\pm$1.6}
& 46.2\scriptsize{$\pm$4.2} & 49.8\scriptsize{$\pm$5.8} & 57.9\scriptsize{$\pm$1.8} & 53.8\scriptsize{$\pm$3.3} & 49.8\scriptsize{$\pm$1.7} & 66.9\scriptsize{$\pm$2.5} & 50.3\scriptsize{$\pm$2.3} & 75.3\scriptsize{$\pm$2.2} \\

\midrule
JacobiConv
& 73.4\scriptsize{$\pm$1.1} & 69.4\scriptsize{$\pm$2.1} & 56.3\scriptsize{$\pm$2.2}
& 70.9\scriptsize{$\pm$1.2} & 68.2\scriptsize{$\pm$1.2} & 55.3\scriptsize{$\pm$1.4}
& 41.7\scriptsize{$\pm$2.9} & 37.0\scriptsize{$\pm$1.7} & 59.3\scriptsize{$\pm$1.2} & 59.7\scriptsize{$\pm$2.1} & 52.9\scriptsize{$\pm$1.9} & 72.9\scriptsize{$\pm$2.5} & 54.3\scriptsize{$\pm$1.4} & 76.2\scriptsize{$\pm$2.5} \\

AutoSGNN
& 44.6\scriptsize{$\pm$1.2} & 44.2\scriptsize{$\pm$1.8} & 53.3\scriptsize{$\pm$2.0}
& 69.9\scriptsize{$\pm$2.3} & 62.4\scriptsize{$\pm$2.7} & 52.3\scriptsize{$\pm$1.2}
& 44.3\scriptsize{$\pm$2.3} & 38.9\scriptsize{$\pm$1.9} & 58.7\scriptsize{$\pm$1.3} & 59.6\scriptsize{$\pm$2.0} & 52.3\scriptsize{$\pm$1.6} & 72.1\scriptsize{$\pm$2.6} & 49.7\scriptsize{$\pm$1.3} & 74.2\scriptsize{$\pm$2.2} \\

\midrule
SFDA\_LLN
& 70.7\scriptsize{$\pm$1.6} & 68.1\scriptsize{$\pm$1.2} & 60.4\scriptsize{$\pm$1.3}
& 67.7\scriptsize{$\pm$1.5} & 66.5\scriptsize{$\pm$2.6} & 58.8\scriptsize{$\pm$2.1}
& 34.7\scriptsize{$\pm$1.1} & 41.5\scriptsize{$\pm$2.1} & 59.2\scriptsize{$\pm$1.1} & 59.9\scriptsize{$\pm$1.5} & 56.9\scriptsize{$\pm$1.9} & 74.2\scriptsize{$\pm$2.0} & 53.5\scriptsize{$\pm$2.2} & 77.3\scriptsize{$\pm$1.8} \\

SF(DA)$^2$
& 72.0\scriptsize{$\pm$1.5} & 69.7\scriptsize{$\pm$1.9} & 62.0\scriptsize{$\pm$1.8}
& 71.7\scriptsize{$\pm$1.9} & 68.8\scriptsize{$\pm$2.3} & 61.9\scriptsize{$\pm$2.0}
& 40.5\scriptsize{$\pm$1.8} & 50.1\scriptsize{$\pm$2.4} & 62.2\scriptsize{$\pm$1.7} & 63.2\scriptsize{$\pm$1.8} & 51.7\scriptsize{$\pm$2.3} & 74.6\scriptsize{$\pm$2.6} & 48.9\scriptsize{$\pm$1.4} & 78.8\scriptsize{$\pm$2.0} \\

NVC\_LLN
& 71.4\scriptsize{$\pm$1.3} & 67.8\scriptsize{$\pm$1.1} & 62.1\scriptsize{$\pm$1.5}
& 70.4\scriptsize{$\pm$2.2} & 67.0\scriptsize{$\pm$2.2} & 62.2\scriptsize{$\pm$1.7}
& 47.2\scriptsize{$\pm$1.6} & 53.0\scriptsize{$\pm$1.4} & 59.2\scriptsize{$\pm$2.1} & 60.0\scriptsize{$\pm$1.8} & 56.6\scriptsize{$\pm$3.1} & 76.7\scriptsize{$\pm$2.0} & 53.8\scriptsize{$\pm$1.8} & 78.2\scriptsize{$\pm$2.0} \\

\midrule
SOGA
& 73.5\scriptsize{$\pm$1.3} & 69.4\scriptsize{$\pm$2.5} & 63.1\scriptsize{$\pm$2.1}
& 72.3\scriptsize{$\pm$2.2} & 69.8\scriptsize{$\pm$1.9} & 63.1\scriptsize{$\pm$1.8}
& 47.4\scriptsize{$\pm$2.8} & 41.5\scriptsize{$\pm$2.2} & 64.4\scriptsize{$\pm$1.3} & 65.4\scriptsize{$\pm$1.8} & 51.5\scriptsize{$\pm$2.0} & 76.9\scriptsize{$\pm$2.5} & 53.3\scriptsize{$\pm$3.2} & 78.9\scriptsize{$\pm$2.3} \\

GraphCTA
& 73.7\scriptsize{$\pm$1.8} & 70.7\scriptsize{$\pm$1.5} & \textbf{65.2\scriptsize{$\pm$1.7}}
& 72.0\scriptsize{$\pm$2.2} & 70.5\scriptsize{$\pm$1.6} & 62.8\scriptsize{$\pm$1.8}
& 50.9\scriptsize{$\pm$2.1} & 45.2\scriptsize{$\pm$1.7} & 60.6\scriptsize{$\pm$1.9} & 61.7\scriptsize{$\pm$1.8} & 56.8\scriptsize{$\pm$3.5} & 77.3\scriptsize{$\pm$2.0} & 54.2\scriptsize{$\pm$2.7} & 79.0\scriptsize{$\pm$2.6} \\

GALA
& 75.1\scriptsize{$\pm$2.1} & 70.8\scriptsize{$\pm$1.6} & 64.0\scriptsize{$\pm$2.1}
& 72.2\scriptsize{$\pm$1.2} & 70.3\scriptsize{$\pm$1.7} & 63.4\scriptsize{$\pm$1.8}
& 51.2\scriptsize{$\pm$4.0} & 47.5\scriptsize{$\pm$3.1} & 65.6\scriptsize{$\pm$1.3} & 66.0\scriptsize{$\pm$2.5} & 61.8\scriptsize{$\pm$2.7} & 77.0\scriptsize{$\pm$2.6} & 57.4\scriptsize{$\pm$2.3} & 78.8\scriptsize{$\pm$1.8} \\

\midrule
\method{}
& \textbf{75.1\scriptsize{$\pm$1.8}} & \textbf{72.0\scriptsize{$\pm$2.1}} & 65.1\scriptsize{$\pm$1.4}
& \textbf{73.2\scriptsize{$\pm$1.7}} & \textbf{71.9\scriptsize{$\pm$1.5}} & \textbf{64.2\scriptsize{$\pm$1.6}}
& \textbf{53.2\scriptsize{$\pm$2.3}} & \textbf{54.7\scriptsize{$\pm$2.7}} & \textbf{66.2\scriptsize{$\pm$1.4}} & \textbf{66.7\scriptsize{$\pm$1.9}} & \textbf{62.9\scriptsize{$\pm$2.1}} & \textbf{77.8\scriptsize{$\pm$2.4}} & \textbf{58.6\scriptsize{$\pm$1.6}} & \textbf{79.6\scriptsize{$\pm$2.1}} \\

\bottomrule
\end{tabular}}
\label{tab:combined_shifts}
\end{table*}

%% file: code/5_experiments.tex
\section{Experiments}

\subsection{Experimental Settings}

\textbf{Datasets.} We evaluate the proposed \method{} on comprehensive benchmarks under three types of domain shifts: (1) Structure-based domain shifts: Structural shifts are simulated on DD~\citep{dobson2003distinguishing}, Mutagenicity~\citep{kazius2005derivation}, NCI1~\citep{wale2008comparison}, FRANKENSTEIN~\citep{orsini2015graph}, and ogbg-molhiv~\citep{hu2021ogblsc} by partitioning graphs according to node and edge densities~\cite{yin2023coco, yin2025dream}. (2) Feature-based domain shifts: \method{} is evaluated on DD, PROTEINS, BZR, BZR\_MD, COX2, and COX2\_MD, where the source and target domains share identical graph structures but exhibit significant discrepancies in feature distributions. (3) Correlation shifts: We further evaluate \method{} on a synthetic Spurious-Motif dataset~\cite{wu2022discovering}, where the predictive relationships between class labels and spurious motif patterns differ between the source and target domains. More details of these datasets are shown in Appendix~\ref{sec:dataset}.

\noindent \textbf{Baselines.} We compare the proposed \method{} with a comprehensive set of competitive baselines on the datasets above. These baselines include: graph kernel methods: WL~\citep{shervashidze2011weisfeiler} and PathNN~\citep{michel2023path}; graph neural networks: GCN~\citep{kipf2017semi}, GIN~\citep{xu2018powerful}, CIN~\citep{bodnar2021weisfeiler}, and GMT~\citep{baek2021accurate}; spectral-based graph neural networks: JacobiConv~\citep{wang2022powerful} and AutoSGNN~\citep{mo2025autosgnn}; source free domain adaptation methods: SFDA\_LLN ~\citep{yi2023source}, SF(DA)$^2$~\citep{hwang2024sf} and NVC-LLN~\citep{xu2025unraveling}; source free graph domain adaptation methods: SOGA~\citep{mao2024source}, GraphCTA~\citep{zhang2024collaborate} and GALA~\citep{luo2024gala}. More details about baselines are introduced in Appendix ~\ref{sec:baselines}.

\noindent \textbf{Implementation Details.} We implement \method{} using PyTorch and conduct all experiments on NVIDIA A100 GPUs. GIN~\citep{xu2018powerful} is adopted as the backbone encoder. \method{} is trained under identical settings with the Adam optimizer, three GNN layers, and a hidden dimension of 128. For the bi-level optimization in \method{}, the inner-loop GNN encoder is updated for 20 steps with a learning rate of $1\times10^{-3}$, identical to that of the outer-loop synthetic basis generator. \method{}-specific hyperparameters are set to $T=20$, $K=20$, $E_{\max}=1.2$, and $\sigma=1.0$ for spectral anchors and inference. The loss balancing coefficients are fixed to $\lambda_1=0.9$ and $\lambda_2=0.5$. We report classification accuracy on TUDataset benchmarks (e.g., DD) and Spurious-Motif, and AUC scores on OGB datasets (e.g., ogbg-molhiv), with results averaged over five runs.

\subsection{Performance Comparison}

We present the results of the proposed \method{} with all baselines under three types of domain shifts on different datasets in Tables \ref{tab:combined_shifts} and \ref{tab:dd_idx}-\ref{tab:hiv_node}. From these tables, we observe that: 
(1) SFDA methods consistently outperform GNN baselines, demonstrating that active domain alignment is superior to simple model transfer. By leveraging unlabeled target data through self-training and consistency regularization, these approaches effectively enhance generalization and mitigate the predictive bias often suffered by source-only models under significant distribution shifts. (2) SF-GDA methods consistently surpass general SFDA approaches, underscoring the necessity of incorporating graph-specific properties during the adaptation process. By explicitly accounting for graph topology and structural dependencies, these methods better preserve inherent semantics and achieve more stable adaptation, especially when facing drastic and unforeseen structural shifts.
(3) The proposed \method{} achieves the best performance in most cases, showing a clear advantage over existing methods. This can be attributed to two main reasons: First, by shifting from model fine-tuning to a data-centric basis distillation, \method{} effectively decouples adaptation from source-specific biases and suppresses the propagation of predictive errors within the target domain. Second, the spectral extrapolation mechanism enables \method{} to cover the full range of graph energy levels, ensuring robust adaptation to diverse topological patterns not seen during source training. More results on other datasets can be found in Appendix~\ref{sec:model performance}.

\begin{figure}[H]
    \begin{subfigure}{0.48\linewidth}
        \centering
        \includegraphics[width=\linewidth]{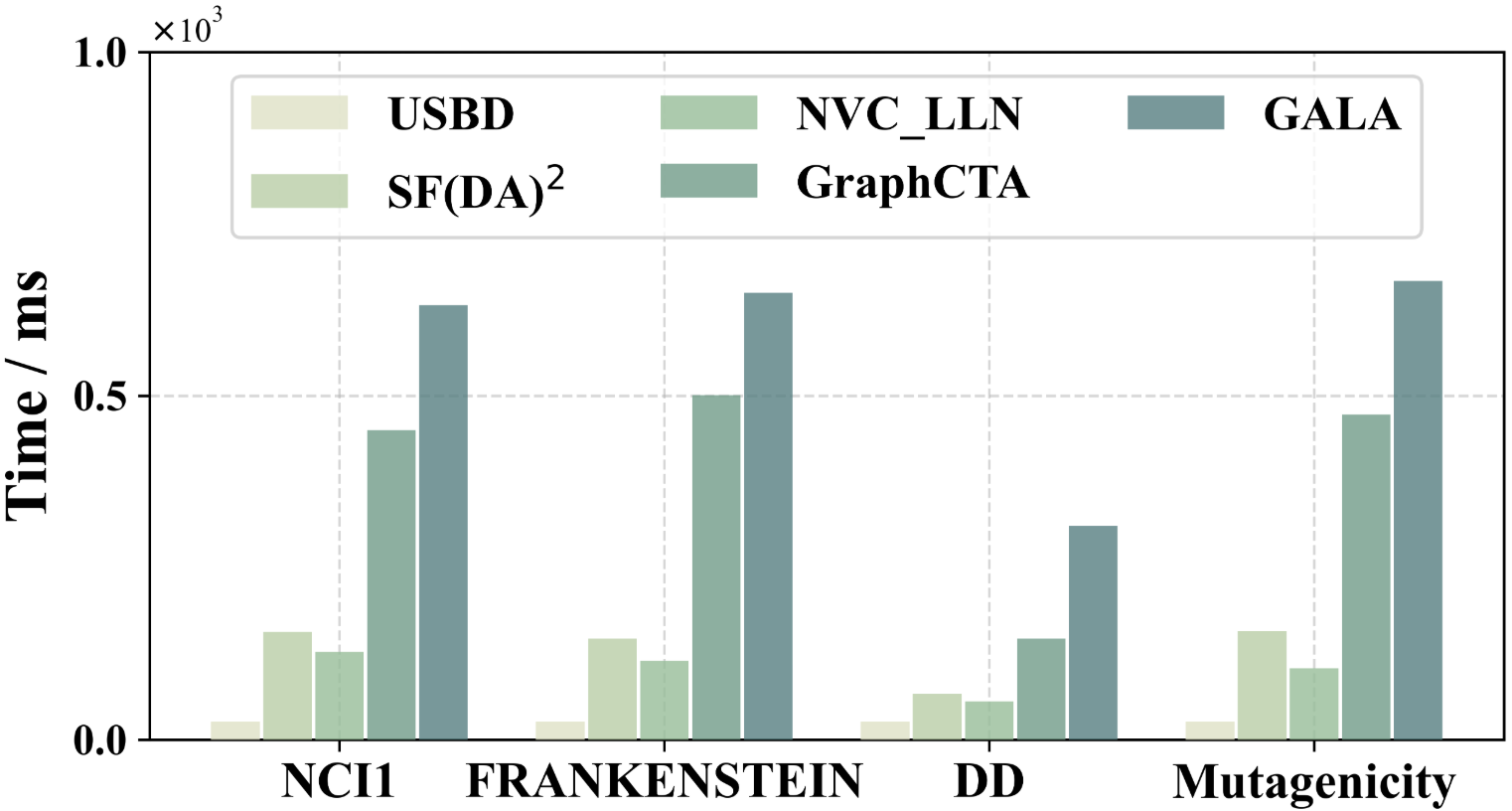}
        \caption{Time consumption}
        \label{fig:adapt_time_ms}
    \end{subfigure}
    \hspace{0.1cm}
    \hfill
    \begin{subfigure}{0.48\linewidth}
        \centering
        \includegraphics[width=\linewidth]{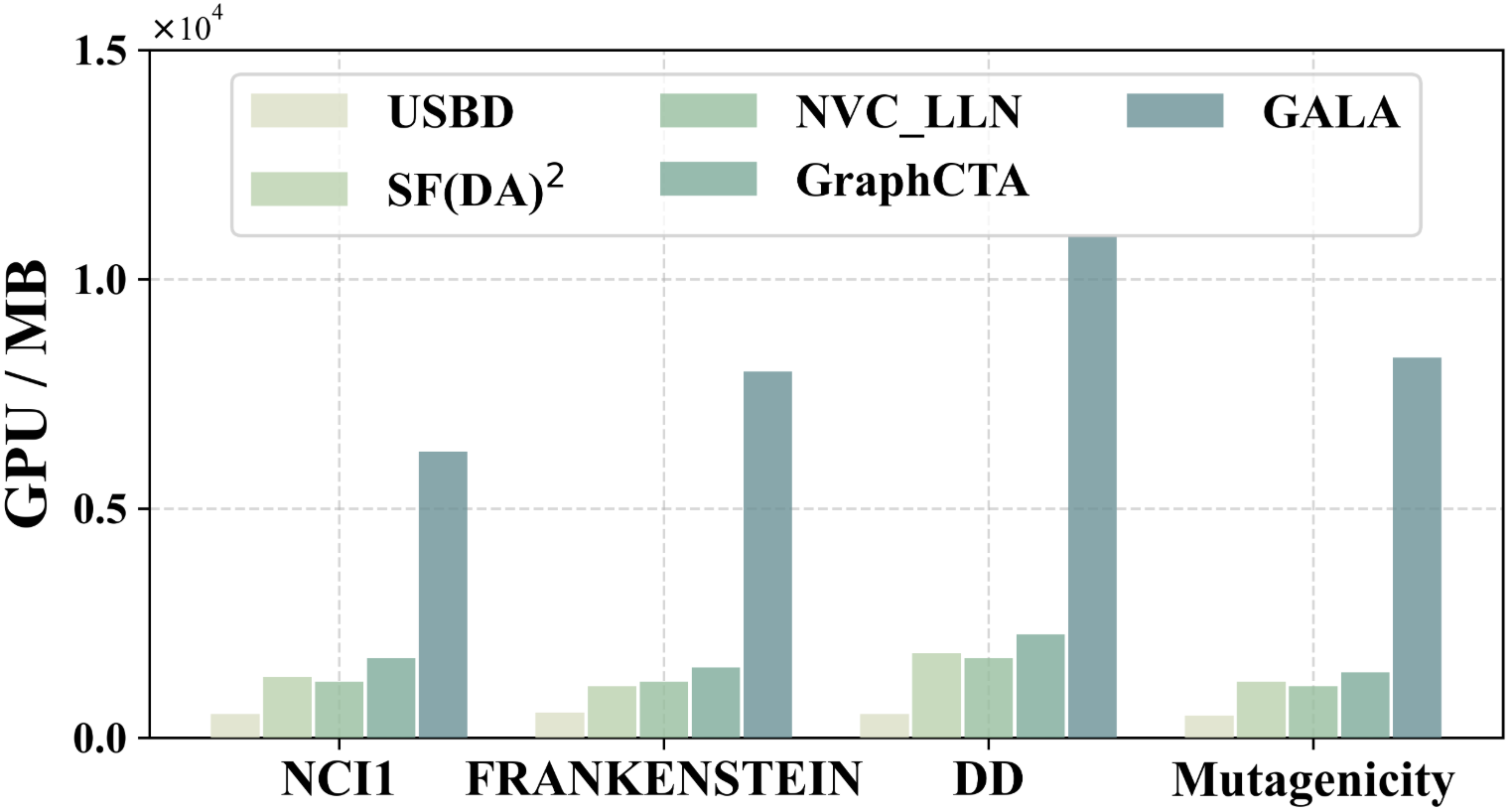}
        \caption{GPU consumption}
        \label{fig:adapt_gpu_mb}
    \end{subfigure}
    \vspace{-0.1cm}
    \caption{Comparison of time and GPU consumption between baseline methods and \method{} during the adaptation stage.}
    \label{fig:adapt_time_gpu}
\end{figure}

\subsection{Efficiency and Generalization Analysis}\label{sec:efficiency}

In this part, we first analyze the efficiency of \method{} during the adaptation stage by comparing it with SF(DA)$^2$, NVC\_LLN, GraphCTA, and GALA in terms of per-epoch training time (ms) and GPU memory consumption (MB). The results are reported in Figure.~\ref{fig:adapt_time_gpu}, where \method{} consistently achieves substantially lower time and memory consumption across all datasets. This superior efficiency stems from our core design: \method{} performs adaptation on a compact, pre-synthesized basis ($|S_{syn}| \ll |D_T|$) rather than directly optimizing on the whole target graphs. This effectively decouples the computational overhead from the target data size, successfully resolving the efficiency bottleneck inherent in most traditional methods. Additionally, to evaluate generalization capability, we conduct a cross-domain study on the Mutagenicity dataset using a challenging "adapt-once, deploy-anywhere" setting. Specifically, models adapted for the M0 $\rightarrow$ M1 task are directly evaluated on a completely unseen target domain, M2, without any further tuning. Figure.~\ref{fig:generalization_tsne_mutag} presents the t-SNE visualizations, where \method{} exhibits significantly more compact intra-class clusters and clearer inter-class separation than existing baselines. This result validates that \method{} successfully learns a universal structural basis capable of handling arbitrary topological shifts, thereby achieving robust and highly scalable generalization across diverse graph environments. More t-SNE visualizations can be found in Appendix~\ref{sec:vis}.

\begin{figure}[t]
    \centering
    \begin{subfigure}[t]{0.23\textwidth}
        \centering
        \includegraphics[width=\linewidth]{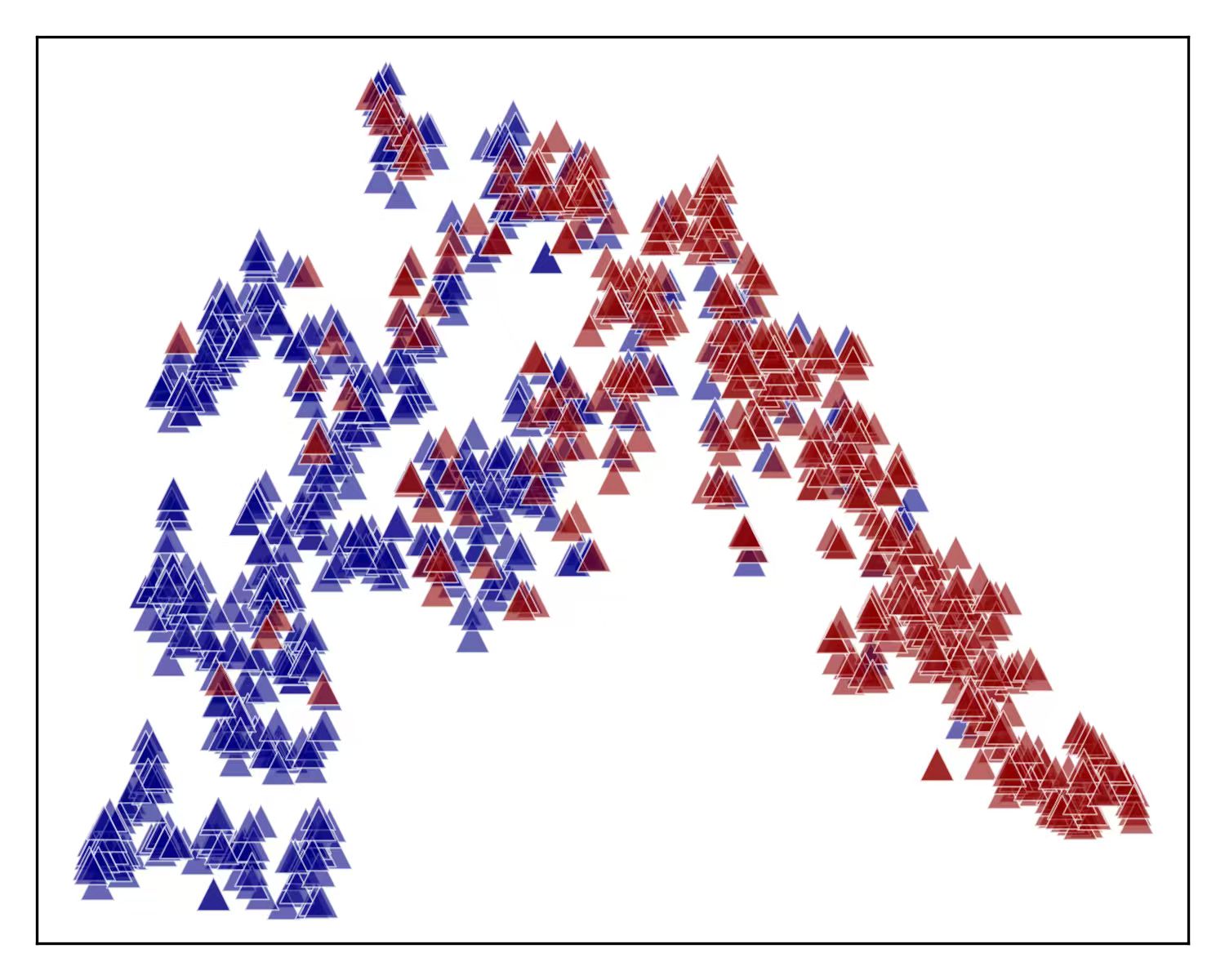}
        \caption{USBD}
        \label{fig:tsne_usbd}
    \end{subfigure}\hfill
    \begin{subfigure}[t]{0.23\textwidth}
        \centering
        \includegraphics[width=\linewidth]{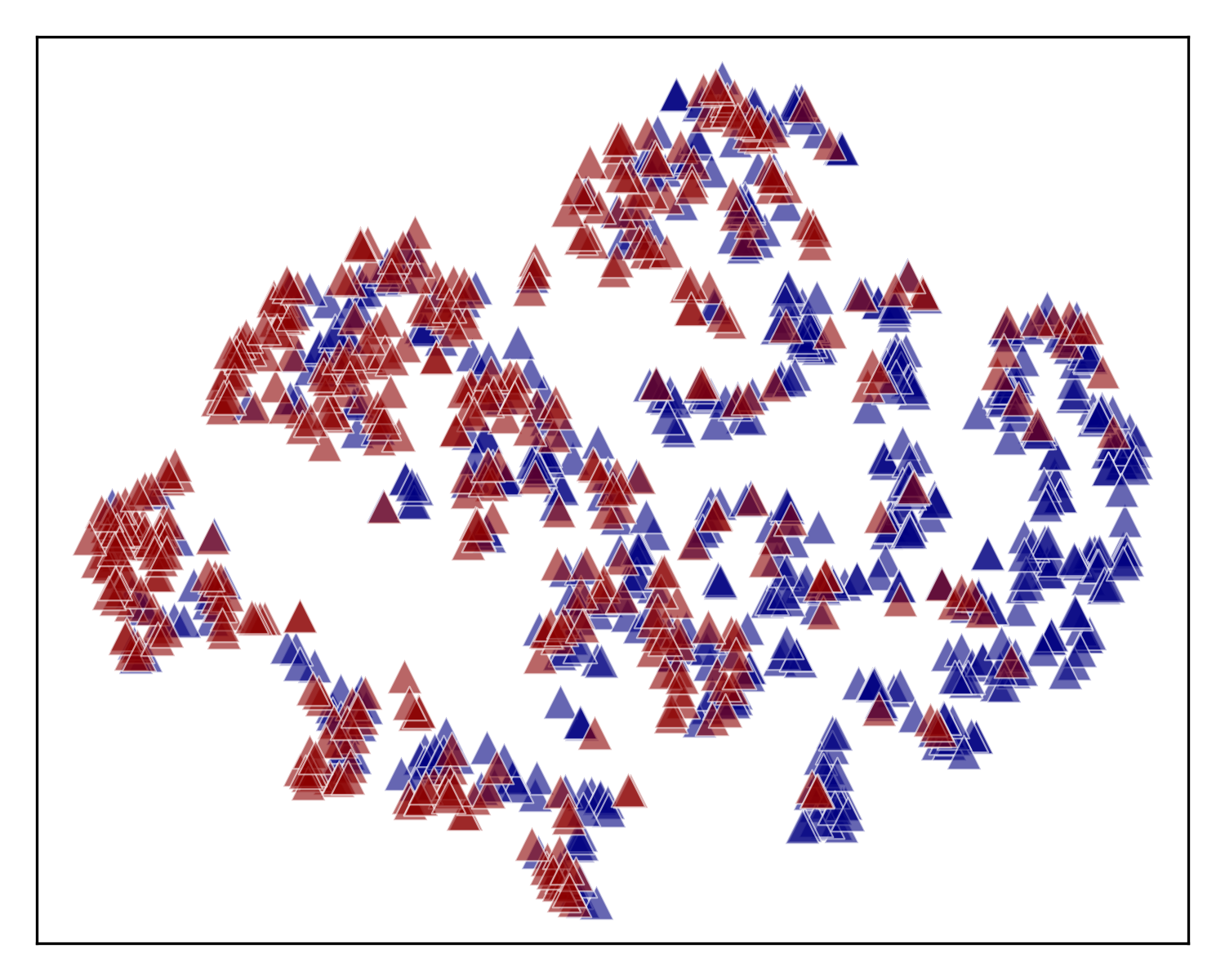}
        \caption{SF(DA)$^2$}
        \label{fig:tsne_sfda2}
    \end{subfigure}
    \vspace{-0.1cm}
    \caption{T-SNE visualizations on M2 of the Mutagenicity dataset for \method{} and baselines trained/adapted on M0$\rightarrow$M1.}
    \vspace{-0.1cm}
    \label{fig:generalization_tsne_mutag}
\end{figure}

\begin{table}[h]
\small
\centering
\caption{The results of ablation studies on the Mutagenicity dataset. \textbf{Bold} results indicate the best performance.}
\vspace{-0.2cm}
\resizebox{0.48\textwidth}{!}{
\begin{tabular}{l|c|c|c|c|c|c}
\toprule
\textbf{Methods}
& M0$\rightarrow$M1 & M1$\rightarrow$M0
& M0$\rightarrow$M2 & M2$\rightarrow$M0
& M0$\rightarrow$M3 & M3$\rightarrow$M0 \\
\midrule 
\method{} w/o SE  & 58.7 & 59.3 & 57.2 & 60.7 & 54.4 & 56.8 \\
\method{} w/o SP  & 69.5 & 69.8 & 67.4 & 69.1 & 59.1 & 68.2 \\
\method{} w/o DI  & 71.9 & 72.5 & 70.8 & 72.4 & 63.5 & 71.5 \\
\method{} w/o AD  & 65.6 & 66.5 & 64.1 & 65.9 & 57.8 & 64.7 \\ 
\midrule
\method{} & \textbf{73.2} & \textbf{73.9} & \textbf{71.9} & \textbf{73.7} & \textbf{64.2} & \textbf{72.8} \\ 
\bottomrule
\end{tabular}
}
\label{tab:ablation_mutag_part}
\vspace{-0.2cm}
\end{table}

\subsection{Ablation Study}~\label{sec:ablation}

To systematically examine the contribution of each key component in \method{}, we conduct ablation studies on four model variants: (1) \method{} w/o SE, which removes the semantic loss $\mathcal{L}_{\text{sem}}$ and discards the bi-level meta-matching constraint aligning the synthesized structural basis with the source decision boundaries; (2) \method{} w/o SP, which excludes the spectral loss $\mathcal{L}_\text{span}$, allowing the generator to synthesize graphs without enforcing coverage of the full Dirichlet energy spectrum; (3) \method{} w/o DI, which removes the diversity loss $\mathcal{L}_{\text{div}}$, thereby eliminating the constraint enforcing pairwise structural orthogonality among the synthesized basis graphs; and (4) \method{} w/o AD, which disables the spectral-aware adaptive weighting mechanism during inference and assigns uniform weights to all synthesized basis graphs irrespective of the spectral fingerprint of target graphs. 

Experimental results are shown in Table~\ref{tab:ablation_mutag_part}. From the table, we observe that: (1) The semantic loss $\mathcal{L}_{\text{sem}}$ is introduced to preserve source-discriminative semantics during basis synthesis via bi-level meta-matching. When this component is removed (\method{} w/o SE), performance degrades consistently across all cases, indicating that without $\mathcal{L}_{\text{sem}}$, the synthesized structural basis no longer remains aligned with the source decision boundaries, which substantially impairs transferability; (2) The spectral loss $\mathcal{L}_{\text{span}}$ and the diversity loss $\mathcal{L}_{\text{div}}$ jointly govern the coverage and distinctiveness of the synthesized basis. Specifically, $\mathcal{L}_{\text{span}}$ enforces coverage over the full Dirichlet energy spectrum, while $\mathcal{L}_{\text{div}}$ discourages structural redundancy among basis graphs. Removing either component (\method{} w/o SP or \method{} w/o DI) leads to consistent performance degradation, indicating that both spectrum-wide extrapolation and structural distinctiveness are essential for robust adaptation under severe structural shifts; (3) The spectral-aware adaptive inference module is designed to align inference with the target domain by activating basis prototypes according to its spectral fingerprint, thereby minimizing the structural covering discrepancy. When this mechanism is disabled (\method{} w/o AD), all basis graphs are weighted uniformly, causing the inference process to ignore spectral mismatches between the target domain and individual prototypes. Consequently, the adapted classifier is optimized on spectrally misaligned bases, leading to consistent performance degradation. More results on other datasets can be found in Appendix~\ref{sec:ablation study}.

\subsection{Sensitivity Analysis}\label{sec:sensitivity}

\begin{figure}[t]
    \begin{subfigure}{0.48\linewidth}
        \centering
    \raisebox{0.3cm}{\includegraphics[width=\linewidth]{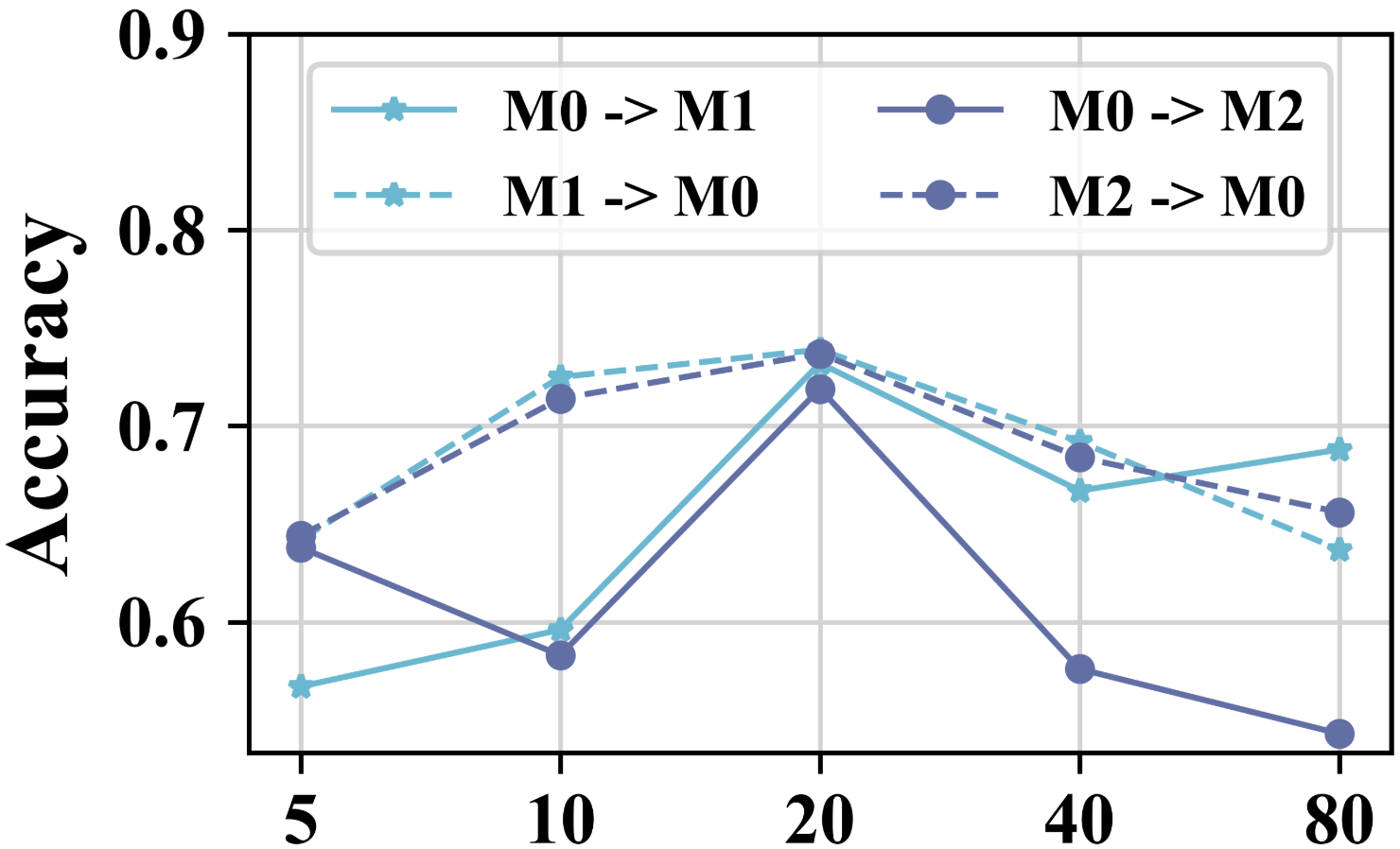}}
        \caption{Synthetic Bases $K$}
        \label{fig:mutag_K}
    \end{subfigure}
    \hspace{0.1cm}
    \hfill
    \begin{subfigure}[b]{0.48\linewidth}\centering\includegraphics[width=\linewidth]{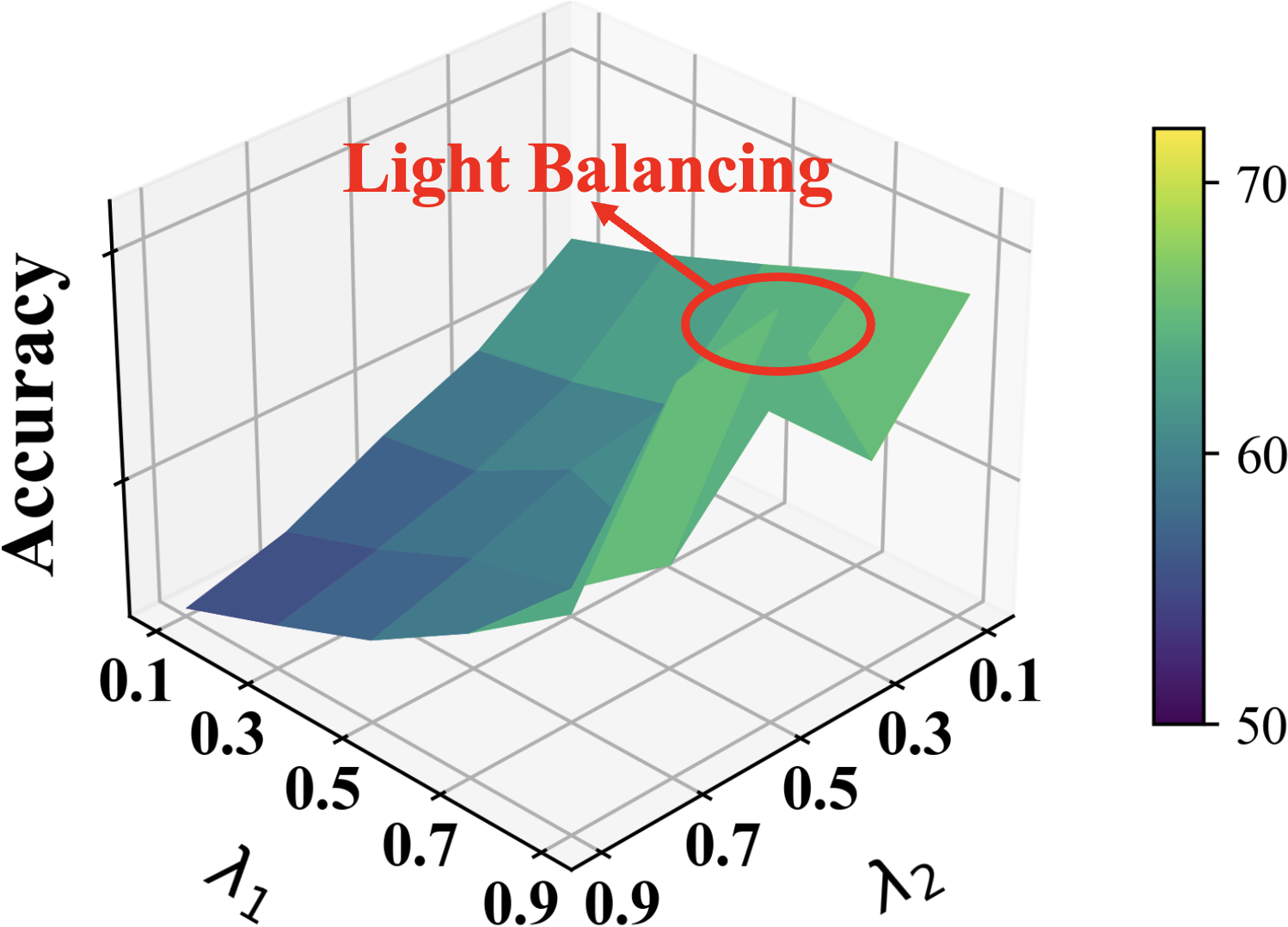}
        \caption{Coefficient ($\lambda_1$, $\lambda_2$)}
        \label{fig:lambda1_lambda2}
    \end{subfigure}
    \vspace{-0.1cm}
    \caption{Sensitivity analysis of the number of synthetic bases $K$ and balance coefficient ($\lambda_1$, $\lambda_2$) on the Mutagenicity dataset.}
    \label{fig:sensitivity}
\end{figure}

We conduct a sensitivity analysis on the number of synthetic bases $K$ and the balance coefficients $(\lambda_1, \lambda_2)$, as shown in Figure.~\ref{fig:sensitivity}. Here, $K$ controls the size of the universal structural basis, while $\lambda_1$ and $\lambda_2$ balance the spectral loss and the diversity loss, respectively, governing spectrum-wide coverage and topological distinctiveness during basis synthesis.

Figure.~\ref{fig:sensitivity} illustrates how these hyperparameters affect the performance of \method{} on the Mutagenicity dataset. We vary the number of synthetic bases $K$ within $\{5, 10, 20, 40, 80\}$, and vary both $\lambda_1$ and $\lambda_2$ within $\{0.1, 0.3, 0.5, 0.7, 0.9\}$. We can find that: (1) As shown in Figure.~\ref{fig:sensitivity}(a), performance improves as $K$ increases from 5 to 20, indicating that expanding the basis enhances structural coverage and facilitates more effective adaptation. When $K$ is further increased, performance exhibits diminishing returns in classification accuracy, suggesting that excessively large bases introduce redundancy and reduce the effectiveness of basis activation. Based on the trade-off between coverage and efficiency, we set $K=20$ as the default standard for all experiments; (2) Figure.~\ref{fig:sensitivity}(b) shows that \method{} achieves optimal performance at $\lambda_1 = 0.9$ and $\lambda_2 = 0.5$. Performance degrades when either coefficient is set too low or too high, indicating that effective adaptation requires a balanced trade-off between spectral coverage, structural diversity, and semantic optimization. More results on other datasets are presented in Appendix~\ref{sec:sensitive analysis}.

\begin{figure}[t]
    \centering
    \begin{subfigure}[t]{0.14\textwidth}
        \centering
        \includegraphics[width=\linewidth]{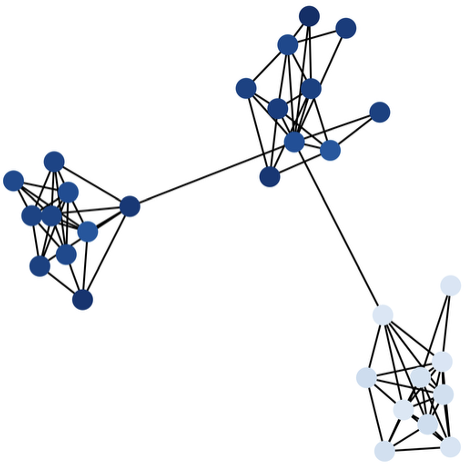}
        \caption{$\mathcal{E}=0.13$}
        \label{fig:graph_1}
    \end{subfigure}
    \hfill
    \begin{subfigure}[t]{0.14\textwidth}
        \centering
        \includegraphics[width=\linewidth]{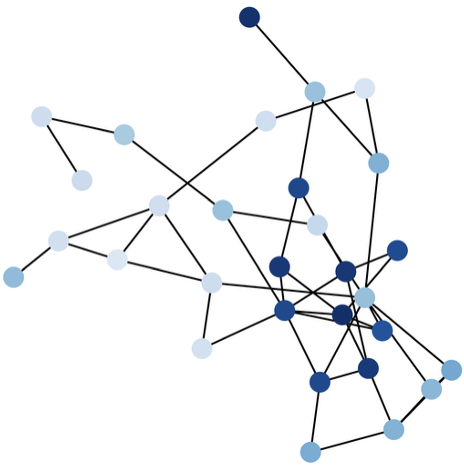}
        \caption{$\mathcal{E}=0.45$}
        \label{fig:graph_3}
    \end{subfigure}
    \hfill
    \begin{subfigure}[t]{0.14\textwidth}
        \centering
        \includegraphics[width=\linewidth]{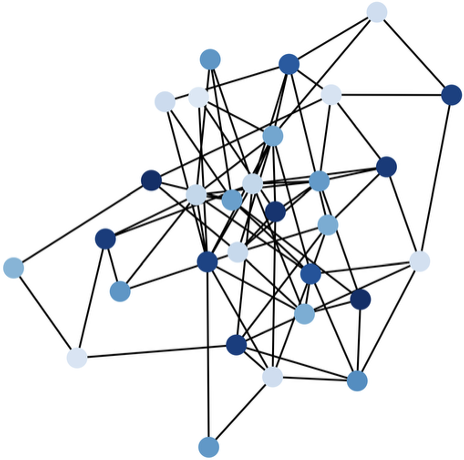}
        \caption{$\mathcal{E}=1.15$}
        \label{fig:graph_5}
    \end{subfigure}
    \vspace{-0.1cm}
    \caption{Visualizations of the universal structural basis with different Dirichlet energy $\mathcal{E}$ on the Mutagenicity dataset.}
    \label{fig:synthetic_basis_energy}
    \vspace{-0.2cm}
\end{figure}

\subsection{Qualitative Analysis of Universal Basis}\label{sec:vis_graph}

In this part, we provide a qualitative assessment of the universal structural basis by visualizing the synthesized prototypes across diverse spectral regimes to evaluate their structural diversity. As illustrated in Figure.~\ref{fig:synthetic_basis_energy}(a), basis graphs characterized by low Dirichlet energy ($\mathcal{E}=0.13$) demonstrate highly clustered and smooth structural formations, which reflect the dominance of low-frequency components typically found in homophilic graph signals. Conversely, the intermediate-energy prototypes ($\mathcal{E}=0.45$) shown in Figure.~\ref{fig:synthetic_basis_energy}(b) exhibit more intricate, mixed structural patterns. These graphs successfully capture transitional topological regimes and facilitate the modeling of more complex node interactions. Furthermore, Figure.~\ref{fig:synthetic_basis_energy}(c) presents high-energy basis graphs ($\mathcal{E}=1.15$), which are defined by irregular connectivity and sparse localized structures, directly corresponding to graph signals dominated by high-frequency components. The distinct progression across the spectral landscape, from Figure.~\ref{fig:synthetic_basis_energy}(a) to Figure.~\ref{fig:synthetic_basis_energy}(c), substantiates that \method{} effectively extrapolates beyond the restricted spectral support of the original source domain. By synthesizing these diverse prototypes, the proposed \method{} constructs a comprehensive structural basis that spans the entire Dirichlet energy spectrum, thereby ensuring robustness against disparate topological shifts. Additional qualitative results and a broader set of synthesized basis graphs have been provided in Appendix~\ref{sec:vis}.

%% file: code/6_conclusion.tex
\section{Conclusion}

In this work, we investigate Source-Free Graph Domain Adaptation (SF-GDA) under severe distributional shifts. We show that existing approaches are fundamentally limited by the inductive and spectral biases embedded in source-pretrained models, which hinder generalization to structurally disparate target graphs. To overcome this limitation, we propose Universal Structural Basis Distillation (\method{}), a data-centric framework that reframes adaptation from fine-tuning biased models to learning a universal structural basis. By distilling source knowledge into a compact set of spectrum-spanning prototypes and coupling it with spectral-aware adaptive inference, \method{} explicitly broadens structural coverage while decoupling adaptation cost from the scale of target data. Extensive experiments across structure, feature, and correlation shift settings demonstrate that \method{} consistently outperforms prior methods across all benchmarks, while maintaining superior computational efficiency. We believe that this universal basis perspective offers a principled direction for robust and privacy-preserving graph adaptation, and lays the groundwork for future exploration of universal representations beyond source-free settings.

%% file: code/7_appendix.tex
\appendix

\section{Proof of Theorem \ref{thm:coverage}}
\label{app:proof_theorem}

\textit{Theorem \ref{thm:coverage} (Universal Spectral Covering Property)}
\textit{Let $\mathcal{H}_{\text{spec}} = [0, E_{\max}]$ be the bounded spectral domain of valid Dirichlet energies. Assume the synthesized basis $\mathcal{S}_{syn} = \{G_k\}_{k=1}^K$ is optimized such that the spectral coverage loss $\mathcal{L}_{\text{span}}$ converges to zero. Let the anchors $\{\mu_k\}_{k=1}^K$ be uniformly distributed over $\mathcal{H}_{\text{spec}}$ with a grid spacing of $\delta = \frac{E_{\max}}{K-1}$. For any unseen target graph $G_T$ with Dirichlet energy $\mathcal{E}(G_T) \in \mathcal{H}_{\text{spec}}$, there exists a basis prototype $G_{k^*} \in \mathcal{S}_{syn}$ such that the spectral approximation error is bounded by:}
\begin{equation}\min_{G_k \in \mathcal{S}_{syn}} | \mathcal{E}(G_T) - \mathcal{E}(G_k) | \le \frac{\delta}{2} + \epsilon. \nonumber
\end{equation}

\begin{proof}The proof analyzes the approximation error by partitioning the spectral domain and accounting for both discretization and optimization residuals.

The spectral domain $\mathcal{H}_{\text{spec}} = [0, E_{\max}]$ is covered by $K$ anchors $\{\mu_k\}_{k=1}^K$ uniformly spaced by $\delta = \frac{E_{\max}}{K-1}$. For any target graph $G_T$ with energy $\mathcal{E}(G_T) \in [0, E_{\max}]$, we define the index of the nearest anchor as:\begin{equation}k^* = \operatorname*{argmin}_{k \in {1,\dots,K}} | \mathcal{E}(G_T) - \mu_k |.
\end{equation}
Due to the uniform distribution of the grid, the distance from any point in the domain to its closest anchor is naturally bounded:
\begin{equation}
\label{step1}
| \mathcal{E}(G_T) - \mu_{k^*} | \le \frac{\delta}{2}.
\end{equation}

The spectral coverage loss $\mathcal{L}_{\text{span}}$ ensures that each synthesized basis $G_k$ aligns with its corresponding anchor $\mu_k$. Given that $\mathcal{L}_{\text{span}}$ converges with an optimization residual $\epsilon$, for each prototype in the basis set, we have:
\begin{equation}
\label{step2}
| \mathcal{E}(G_k) - \mu_k | \le \epsilon, \quad \forall k \in {1, \dots, K}.\end{equation}

We now bound the distance between the target graph energy $\mathcal{E}(G_T)$ and the energy of the selected basis prototype $\mathcal{E}(G_{k^*})$. By applying the Triangle Inequality, we obtain:
\begin{equation}
|\mathcal{E}(G_T) - \mathcal{E}(G_k)| \leq |\mathcal{E}(G_T) - \mu_k | + | \mu_k - \mathcal{E}(G_k)|.
\end{equation}
Substituting the bounds derived in Eq.~\eqref{step1} and Eq.~\eqref{step2} into the above inequality leads to:
\begin{equation}
| \mathcal{E}(G_T) - \mathcal{E}(G_{k^*}) | \le \frac{\delta}{2} + \epsilon.
\end{equation}

Since $G_{k^*}$ is a member of the synthesized set $\mathcal{S}_{syn}$, the minimum error over the entire set is bounded by the error of this specific instance:
\begin{equation}
\min_{G_k \in \mathcal{S}{syn}} | \mathcal{E}(G_T) - \mathcal{E}(G_k) | \le | \mathcal{E}(G_T) - \mathcal{E}(G_{k^*}) | \le \frac{\delta}{2} + \epsilon.
\end{equation}
The target generalization error is thus theoretically controlled by the basis density $K$ and optimization precision $\epsilon$.\end{proof}

\input{figure/data_vis}
\input{table/dataset}

\section{Proof of Theorem \ref{thm:generalization}}
\label{app:proof_theorem3}

Theorem \ref{thm:generalization} (Generalization Bound via Universal Covering)
\textit{Let $f$ be the graph encoder and $h$ be the hypothesis (classifier). Let $\mathcal{R}_{\mathcal{T}}(h \circ f)$ denote the expected risk on the target domain, and $\hat{\mathcal{R}}_{\mathcal{S}_{syn}}(h \circ f)$ denote the empirical risk on the synthesized universal basis $\mathcal{S}_{syn}$.
Assume the loss function $\ell$ is $L_{\ell}$-Lipschitz continuous with respect to the graph representation, and the encoder $f$ is $L_{f}$-Lipschitz continuous with respect to the Dirichlet energy spectrum.
Then, with probability at least $1-\delta$, the target risk is bounded by:
\begin{equation}
    \mathcal{R}_{\mathcal{T}}(h \circ f) \leq \hat{\mathcal{R}}_{\mathcal{S}_{syn}}(h \circ f) + K_{\text{Lip}} \cdot \mathbb{E}_{G \sim \mathcal{D}_T} \left[ \min_{G_k \in \mathcal{S}_{syn}} | \mathcal{E}(G) - \mathcal{E}(G_k) | \right] + \lambda,\nonumber
\end{equation}
where $K_{\text{Lip}} = L_{\ell} L_{f}$ is the composite Lipschitz constant, and $\lambda$ represents the minimal combined risk of the optimal hypothesis.}


\begin{proof}
Let $\mathcal{D}_T$ be the target domain distribution. Let $h \circ f$ denote the composition of the encoder and the classifier. The expected target risk is defined as $\mathcal{R}_{\mathcal{T}} = \mathbb{E}_{G \sim \mathcal{D}_T}[\ell(h(f(G)), y)]$.
We define a projection mapping $\pi: \mathcal{D}_T \to \mathcal{B}$ that maps any target graph $G$ to its structurally nearest neighbor in the universal basis $\mathcal{B}$:
\begin{equation}
    \pi(G) = \operatorname*{arg\,min}_{\hat{G} \in \mathcal{B}} | \mathcal{E}(G) - \mathcal{E}(\hat{G}) |.
\end{equation}
We invoke the Lipschitz continuity assumptions stated in the theorem:
1. The loss function $\ell$ is $L_{\ell}$-Lipschitz continuous with respect to the feature representation: $|\ell(h(\mathbf{z}_1), y) - \ell(h(\mathbf{z}_2), y)| \leq L_{\ell} \|\mathbf{z}_1 - \mathbf{z}_2\|$.
2. The encoder $f$ is $L_{f}$-Lipschitz continuous with respect to the Dirichlet energy spectrum: $\|f(G_1) - f(G_2)\| \leq L_{f} |\mathcal{E}(G_1) - \mathcal{E}(G_2)|$.

We start by decomposing the loss for an arbitrary target sample $G \sim \mathcal{D}_T$. By adding and subtracting the loss term of its nearest basis prototype $\pi(G)$, we apply the Triangle Inequality:
\begin{equation}
    \ell(h(f(G)), y) \leq \ell(h(f(\pi(G))), y) + | \ell(h(f(G)), y) - \ell(h(f(\pi(G))), y) |.
\end{equation}

We focus on the second term, which represents the deviation caused by structural shifting. Applying the Lipschitz continuity of the loss function $\ell$:
\begin{equation}
    | \ell(h(f(G)), y) - \ell(h(f(\pi(G))), y) | \leq L_{\ell} \| f(G) - f(\pi(G)) \|.
\end{equation}
Next, we apply the Lipschitz continuity of the encoder $f$ with respect to the structural energy:
\begin{equation}
    \| f(G) - f(\pi(G)) \| \leq L_{f} | \mathcal{E}(G) - \mathcal{E}(\pi(G)) |.
\end{equation}
Combining these inequalities, the single-sample loss is bounded by:
\begin{equation}
    \ell(h(f(G)), y) \leq \ell(h(f(\pi(G))), y) + L_{\ell} L_{f} \min_{\hat{G} \in \mathcal{B}} | \mathcal{E}(G) - \mathcal{E}(\hat{G}) |.
\end{equation}

We now take the expectation over the target distribution $G \sim \mathcal{D}_T$. The left side becomes the expected target risk $\mathcal{R}_{\mathcal{T}}$. The right side decomposes into the expected risk on the projected basis samples and the expected covering error:
\begin{equation}
    \mathcal{R}_{\mathcal{T}} \leq \mathbb{E}_{G \sim \mathcal{D}_T} [\ell(h(f(\pi(G))), y)] + L_{\ell} L_{f} \mathbb{E}_{G \sim \mathcal{D}_T} \left[ \min_{\hat{G} \in \mathcal{B}} | \mathcal{E}(G) - \mathcal{E}(\hat{G}) | \right].
\end{equation}

The term $\mathbb{E}_{G \sim \mathcal{D}_T} [\ell(h(f(\pi(G))), y)]$ represents the expected risk on the discrete set of basis prototypes weighted by the target density. By standard statistical learning theory (e.g., McDiarmid's inequality or VC-dimension bounds), the expected risk on the basis is bounded by the empirical risk $\hat{\mathcal{R}}_{\mathcal{B}}$ plus a generalization complexity term (denoted as $\text{const}$ or $\epsilon_{\text{gen}}$) with high probability $1-\delta$:
\begin{equation}
    \mathbb{E}_{G \sim \mathcal{D}_T} [\ell(h(f(\pi(G))), y)] \leq \hat{\mathcal{R}}_{\mathcal{B}}(h \circ f) + \text{const}.
\end{equation}

Substituting Step 3 into Step 2, we arrive at the final bound:
\begin{equation}
    \mathcal{R}_{\mathcal{T}}(h \circ f) \leq \hat{\mathcal{R}}_{\mathcal{B}}(h \circ f) + L_{\ell} L_{f} \cdot \mathbb{E}_{G \sim \mathcal{D}_T} \left[ \min_{\hat{G} \in \mathcal{B}} | \mathcal{E}(G) - \mathcal{E}(\hat{G}) | \right] + \text{const}.
\end{equation}
This confirms that the target generalization error is controlled by the performance on the universal basis and the structural covering discrepancy.
\end{proof}

\section{Dataset}\label{sec:dataset}

\subsection{Dataset Description}

We conduct extensive experiments on a variety of datasets. The statistics of the datasets are summarized in Table \ref{tab:dataset}. The detailed descriptions of these dataset are provided as follows:

(1) For structure-based domain shifts:

\begin{itemize}

\item \textbf{DD.} The DD dataset~\citep{dobson2003distinguishing} consists of 1,178 graphs representing protein structures, where nodes correspond to amino acids and edges encode spatial or chemical proximity. Following ~\cite{yin2025dream}, we divide the dataset into four subsets, denoted as D0, D1, D2, and D3, according to graph-level statistics, specifically edge density and node density.

\item \textbf{NCI1.} The NCI1 dataset~\cite{wale2008comparison} comprises 4,100 molecular graphs, where nodes correspond to atoms and edges represent chemical bonds. Each graph is associated with a binary label indicating whether the molecule inhibits cancer cell growth. Following the DD dataset, we divide the dataset into four subsets, denoted as N0, N1, N2, and N3, based on graph-level statistics, specifically edge density and node density.

\item \textbf{Mutagenicity.} 
The Mutagenicity dataset~\cite{kazius2005derivation} comprises 4,337 molecular graphs, where nodes correspond to atoms and edges represent chemical bonds. Each graph is associated with a binary label indicating whether the compound is mutagenic. Following the DD dataset, we divide the dataset into four subsets, denoted as M0, M1, M2, and M3, based on graph-level statistics, specifically edge density and node density.

\item \textbf{FRANKENSTEIN.} The FRANKENSTEIN dataset~\cite{orsini2015graph} consists of 4,337 molecular graphs, where nodes represent atoms and edges denote chemical bonds. Each graph is labeled by the molecule’s biological activity. Following the DD dataset, we divide the dataset into four subsets, denoted as F0, F1, F2, and F3, based on graph-level statistics, specifically edge density and node density.

\item \textbf{ogbg-molhiv.} The ogbg-molhiv dataset~\cite{hu2021ogblsc} comprises 41,127 molecular graphs, where nodes correspond to atoms and edges represent chemical bonds. Each graph is associated with a binary label indicating whether the molecule exhibits HIV inhibitory activity. Following the DD dataset, we divide the dataset into four subsets, denoted as H0, H1, H2, and H3, based on graph-level statistics, specifically edge density and node density.
\end{itemize}

(2) For feature-based domain shifts:

\begin{itemize}
\item \textbf{PROTEINS.} The PROTEINS dataset~\cite{dobson2003distinguishing} consists of 1,113 protein graphs, each annotated with a binary label indicating whether the protein is an enzyme. In each graph, nodes correspond to amino acids, and edges connect amino acids that are within 6~\AA{} of each other along the sequence. Compared to the DD dataset, graphs in PROTEINS are generally smaller and sparser, resulting in reduced structural complexity while preserving comparable semantic labels.

\item \textbf{COX2.} The COX2 dataset~\cite{dobson2003distinguishing} consists of 467 molecular graphs, while COX2\_MD contains 303 structurally modified molecular graphs. In both datasets, nodes represent atoms and edges correspond to chemical bonds. The COX2\_MD dataset introduces controlled structural variations relative to COX2 while preserving the underlying semantic labels.

\item \textbf{BZR.} The BZR dataset~\cite{dobson2003distinguishing} consists of 405 molecular graphs, while BZR\_MD contains 306 structurally modified graphs derived from BZR. In both datasets, nodes represent atoms and edges correspond to chemical bonds. The BZR\_MD dataset introduces controlled structural variations to simulate domain shifts while preserving consistent label semantics.

\end{itemize}

(3) For correlation shifts:

\begin{itemize}

\item \textbf{Spurious-Motif.} The Spurious-Motif dataset~\cite{wu2022discovering} is a synthetic dataset designed to evaluate model robustness under structural correlation shifts. Each graph is constructed by attaching a label-determining motif (e.g., house, cycle) to a base graph (e.g., tree, ladder), which serves as the environmental background. Node features are intentionally uninformative, forcing models to rely exclusively on topological structure. We construct two variants within this dataset, each containing 3,000 graphs: a biased variant, denoted as Spurious-Motif\_bias, in which specific motifs are strongly correlated with particular base graphs, inducing spurious structural shortcuts; and an unbiased variant, denoted as Spurious-Motif, where motifs and base graphs are combined uniformly at random.

\end{itemize}

\subsection{Data Processing}

For datasets from TUDataset~\footnote{https://chrsmrrs.github.io/datasets/} (e.g., Mutagenicity and FRANKENSTEIN), we adopt the standard preprocessing and normalization procedures provided by PyTorch Geometric~\footnote{https://pyg.org/} for loading dataset. For datasets from the Open Graph Benchmark (OGB)~\footnote{https://ogb.stanford.edu/}, such as ogbg-molhiv, we follow the official OGB preprocessing and normalization protocols. For the synthetic Spurious-Motif dataset, we follow the generation and processing protocols described in~\cite{wu2022discovering} and convert the synthesized graphs into standard PyTorch Geometric data objects. Node features are initialized as uninformative constant vectors with dimensionality $d=1$, ensuring that models rely exclusively on graph topology. The dataset is further organized into biased and unbiased variants, which serve as the source and target domains, respectively.

\section{Baselines}\label{sec:baselines}

\subsection{Baseline Description}

In this part, we introduce the details of the compared baselines as follows:

(1) \textbf{Graph kernel methods.} We compare \method{} with two graph kernel methods: 
\begin{itemize}
    \item \textbf{WL subtree}: WL subtree \citep{shervashidze2011weisfeiler} is a hierarchical kernel method that iteratively aggregates node neighborhood labels to capture graph structural similarities efficiently and discriminatively.
    \item \textbf{PathNN}: PathNN~\citep{michel2023path} is an expressive graph neural architectures that aggregate information over simple paths rather than immediate neighborhoods, enabling accurate modeling of higher-order structural patterns and improving representational power beyond standard message-passing GNNs.
    \end{itemize}

(2) \textbf{General Graph Neural Networks.} We compare \method{} with four general graph neural networks: 
    \begin{itemize}
        \item \textbf{GCN}: GCN \citep{kipf2017semi} is a graph neural network that performs localized feature aggregation by applying convolution-like operations over graph neighborhoods, enabling effective learning of node representations through iterative message passing and feature smoothing. 
        \item \textbf{GIN}: GIN \citep{xu2018powerful} is a graph neural network that employs sum aggregation and multilayer perceptrons to achieve maximal expressive power among message-passing GNNs, matching the discriminative capability of the Weisfeiler–Lehman graph isomorphism test.
        \item \textbf{CIN}: CIN \citep{bodnar2021weisfeiler} is is a cellular graph neural architecture that generalize the Weisfeiler–Lehman framework by operating on higher-dimensional cells, enabling more expressive and topology-aware representations than standard message-passing GNNs.
        \item \textbf{GMT}: GMT \citep{baek2021accurate} is a graph neural network framework that learns accurate graph-level representations by modeling node embeddings as multisets and applying expressive, permutation-invariant pooling operators.
    \end{itemize}

(3) \textbf{Spectral-based Graph Neural Networks.} We compare \method{} with two spectral-based graph neural networks: 

\begin{itemize}
    \item \textbf{JacobiConv}: JacobiConv~\cite{wang2022powerful} is a graph representation model that operates in the graph spectral domain, analyzing their expressive power in capturing structural patterns through eigenvalue-based filtering and frequency-aware transformations.
    \item \textbf{AutoSGNN}: AutoSGNN~\cite{mo2025autosgnn} is a spectral graph neural network framework that automatically discovers effective propagation mechanisms by learning adaptive spectral filters, improving representation quality without manual design of graph convolution operators.
\end{itemize}

(4) \textbf{General Source Free Domain Adaptation methods.} We compare \method{} with three general source free domain adaptation methods: 

\begin{itemize}
\item \textbf{SFDA\_LLN:} SFDA\_LLN~\citep{yi2023source} addresses source-free domain adaptation under noisy labels by leveraging self-training and noise-robust learning objectives, enabling effective target-domain adaptation without access to source data while mitigating label noise.
    \item \textbf{SF(DA)$^2$}: SF(DA)$^2$~\citep{hwang2024sf} is a source-free domain adaptation method that leverages data augmentation to generate diverse target views, enabling effective adaptation by enforcing prediction consistency without requiring access to source-domain data.
    \item \textbf{NVC-LLN:} NVC-LLN~\citep{xu2025unraveling} investigates source-free domain adaptation under label noise, combining theoretical analysis with practical algorithms to characterize noise effects and design robust adaptation strategies that improve performance on unlabeled target domains.
\end{itemize}

(5) \textbf{Source Free Graph Domain Adaptation methods.} We compare \method{} with three source free graph domain adaptation methods: 

\begin{itemize}
    \item \textbf{SOGA:} SOGA~\citep{mao2024source}  is a framework that adapts graph neural networks to target domains without source data by leveraging self-training, pseudo-labeling, and structure-aware alignment to mitigate distribution shifts across graph domains.
    \item \textbf{GraphCTA:} GraphCTA~\citep{zhang2024collaborate} is a source-free graph domain adaptation method that performs bi-directional adaptation between a target model and auxiliary counterparts, using mutual consistency and pseudo-label refinement to progressively align representations under cross-domain graph shifts.
    \item \textbf{GALA:} GALA~\citep{luo2024gala} is a source-free graph domain adaptation method that aligns target-domain representations via graph diffusion and a jigsaw-based self-supervised task, enabling robust adaptation without access to source data.
\end{itemize}

\subsection{Implementation Details}

We implement all baseline methods and conduct all experiments on NVIDIA A100 GPUs to ensure fair comparisons. All methods are evaluated under the same hyperparameter configuration as the proposed \method{}, using the Adam optimizer with a learning rate of $1\times10^{-3}$, a weight decay of $1\times10^{-12}$, a hidden embedding dimension of 128, and three GNN layers. Following ~\citep{yin2022deal, yin2023coco}, all models are trained using labeled source-domain data and adapted/evaluated on unlabeled target-domain data. We report classification accuracy on TUDataset benchmarks (e.g., DD and FRANKENSTEIN) and the Spurious-Motif dataset, and AUC scores on OGB datasets (e.g., ogbg-molhiv), with results averaged over five independent runs.

\section{Complexity Analysis}

In this part, we analyze the computational complexity of the proposed \method{}, which comprises the offline bi-level universal basis distillation and the online spectral-aware inference. Let $K$ denote the number of synthetic basis graphs, each with $N'$ nodes and $E'$ edges, $d$ the embedding dimension, and $L$ the number of GNN layers. During the offline distillation phase, the inner loop updates the proxy model via cross-entropy loss, incurring a cost of $\mathcal{O}(I_{inner} \cdot K \cdot L \cdot (E' \cdot d + N' \cdot d^2))$, where $I_{inner}$ is the number of inner steps. In the outer loop, the spectral span loss requires calculating Dirichlet energies with a complexity of $\mathcal{O}(K \cdot (E' \cdot d + N' \cdot d^2))$, while the structural diversity loss involves computing pairwise Gromov-Wasserstein (GW) distances; assuming an implementation via entropic regularization with $I_{GW}$ iterations, this introduces an additional overhead of $\mathcal{O}(K^2 \cdot I_{GW} \cdot (N')^2)$. Although the quadratic complexity of GW with respect to the basis size is computationally intensive, this operation is only performed exclusively during the one-time pre-training phase. During the inference stage, computing the spectral fingerprint of target graphs and performing spectral weighting require $\mathcal{O}(E_T \cdot d + N_T \cdot d^2 + K \cdot d)$ time, while the final prediction on a target graph with $N_T$ nodes and $E_T$ edges takes $\mathcal{O}(L \cdot (E_T \cdot d + N_T \cdot d^2))$.

\section{More experimental results}

\subsection{More performance comparison}\label{sec:model performance}

In this section, we present additional experimental results comparing the proposed \method{} with all baseline methods across various datasets, as summarized in Tables~\ref{tab:dd_idx}–\ref{tab:hiv_node}. The results show that \method{} consistently outperforms the baselines in most transfer settings, further demonstrating its effectiveness and robustness.

\subsection{More Visualizations}\label{sec:vis}

In this part, we provide additional visualizations to further support the analyses presented in the main paper. Specifically, Figure~\ref{fig:generalization_tsne_mutag_more} reports the t-SNE visualizations of additional baselines mentioned in Section~\ref{sec:efficiency}, including NVC\_LLN, GraphCTA, and GALA. In addition, Figure~\ref{fig:vis_graph} presents the complete set of synthesized Universal Structural Basis graphs introduced in Section~\ref{sec:vis_graph}. These visualizations corroborate the findings in Section~\ref{sec:vis_graph}, confirming that the learned basis spans diverse spectral regimes and structural patterns.

\subsection{More Ablation study}\label{sec:ablation study}

To validate the effectiveness of each component in \method{}, we further conduct ablation studies on the DD, NCI1, FRANKENSTEIN and ogbg-molhiv datasets. Specifically, we evaluate four ablated variants of \method{} by removing key modules, including \method{} w/o SE, \method{} w/o SP, \method{} w/o DI, and \method{} w/o AD. The corresponding experimental results are reported in Tables~\ref{tab:ablation_dd}, \ref{tab:ablation_frankenstein}, \ref{tab:ablation_nci1} and \ref{tab:ablation_molhiv}. Overall, the observed performance trends are consistent with those discussed in Section~\ref{sec:ablation}, further confirming the contribution of each component to the overall effectiveness of \method{}.

\subsection{More Sensitivity Analysis}\label{sec:sensitive analysis}

In this part, we further analysis the sensitivity of \method{} to the number of synthetic bases $K$ and balance coefficient ($\lambda_1$, $\lambda_2$) on the DD, NCI1, FRANKENSTEIN and ogbg-molhiv datasets. The results, summarized in Figure~\ref{fig:hyper_coef} and ~\ref{fig:hyper_k}, exhibit trends consistent with the analysis presented in Section~\ref{sec:sensitivity}.

\begin{figure*}[t]
    \centering
    \begin{subfigure}[t]{0.32\textwidth}
        \centering
        \includegraphics[width=\linewidth]{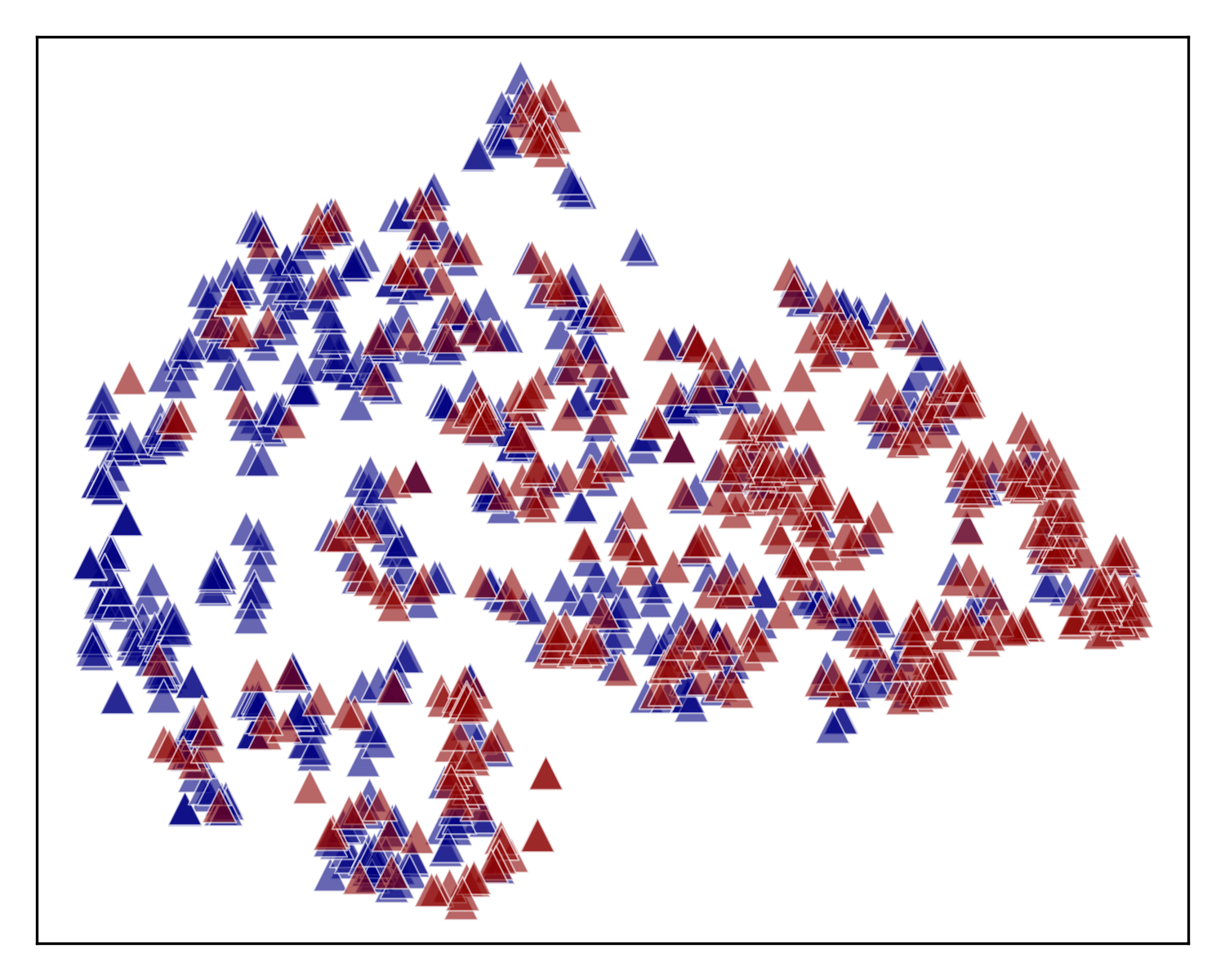}
        \caption{NVC\_LLN}
    \end{subfigure}
    \hfill
    \begin{subfigure}[t]{0.32\textwidth}
        \centering
        \includegraphics[width=\linewidth]{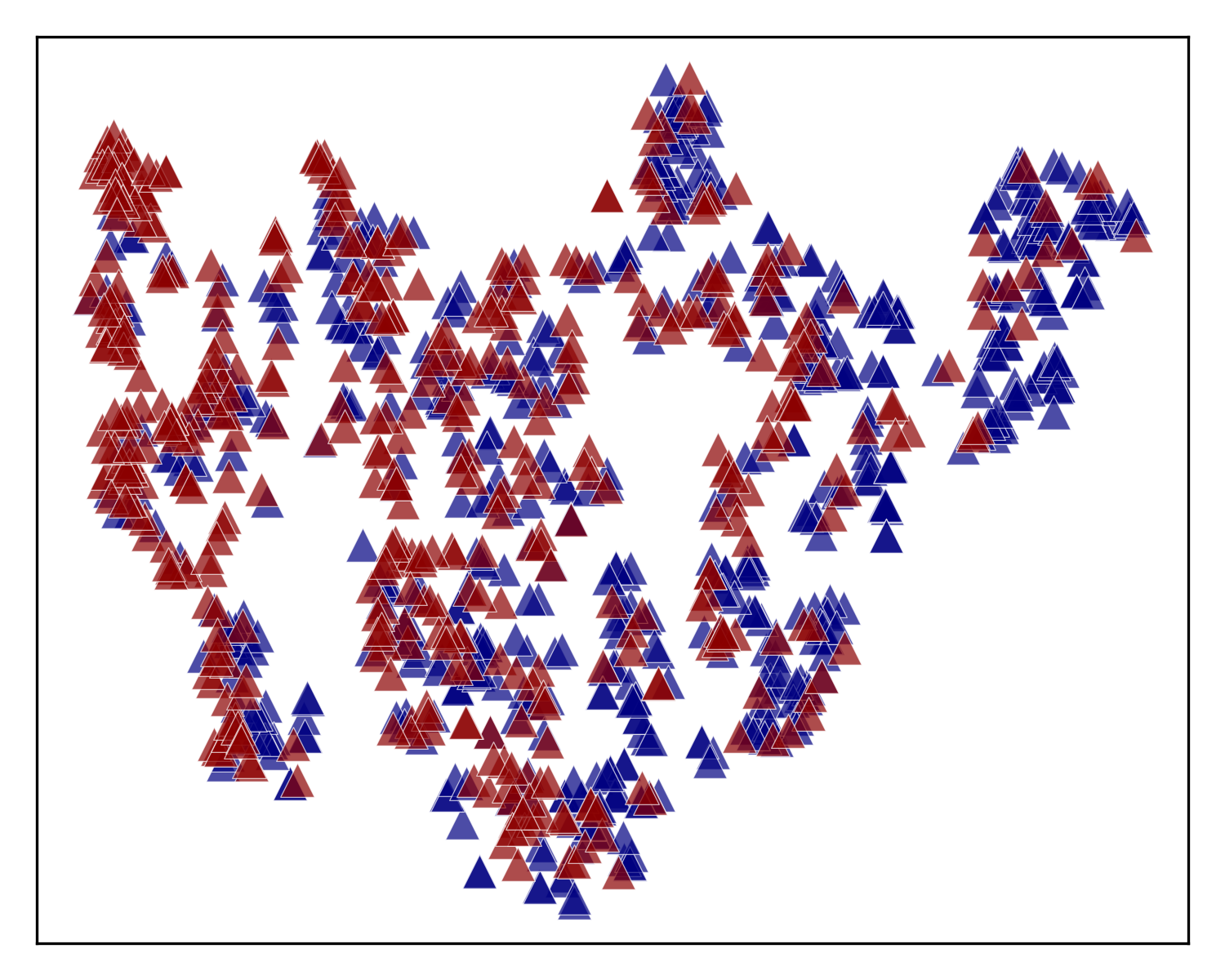}
        \caption{GraphCTA}
    \end{subfigure}
    \hfill
    \begin{subfigure}[t]{0.32\textwidth}
        \centering
        \includegraphics[width=\linewidth]{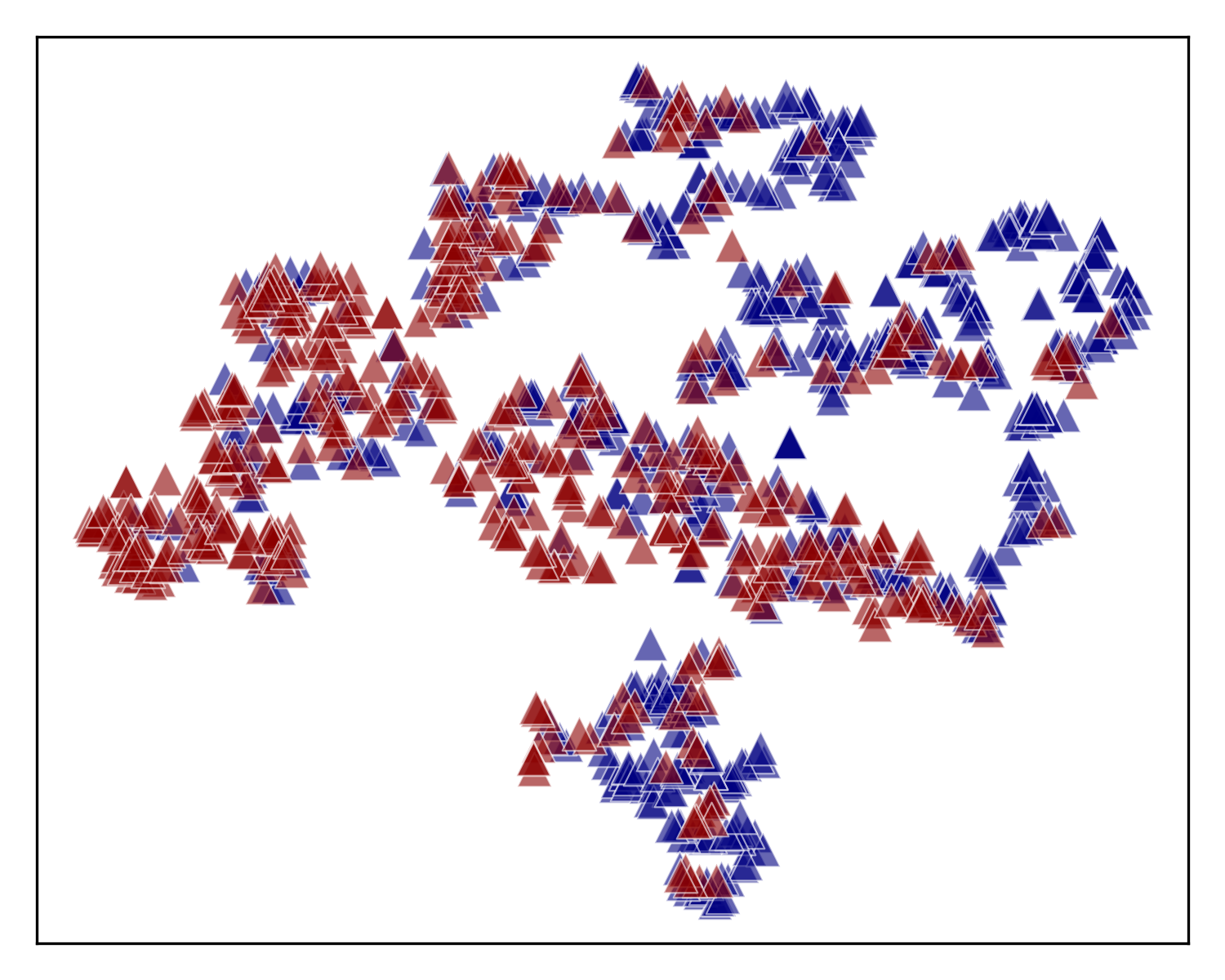}
        \caption{GALA}
    \end{subfigure}
    \caption{T-SNE visualizations on M2 of the Mutagenicity dataset for additional baselines trained/adapted on M0$\rightarrow$M1.}
    \label{fig:generalization_tsne_mutag_more}
\end{figure*}

\begin{figure*}[t]
\centering
\includegraphics[scale=0.58]{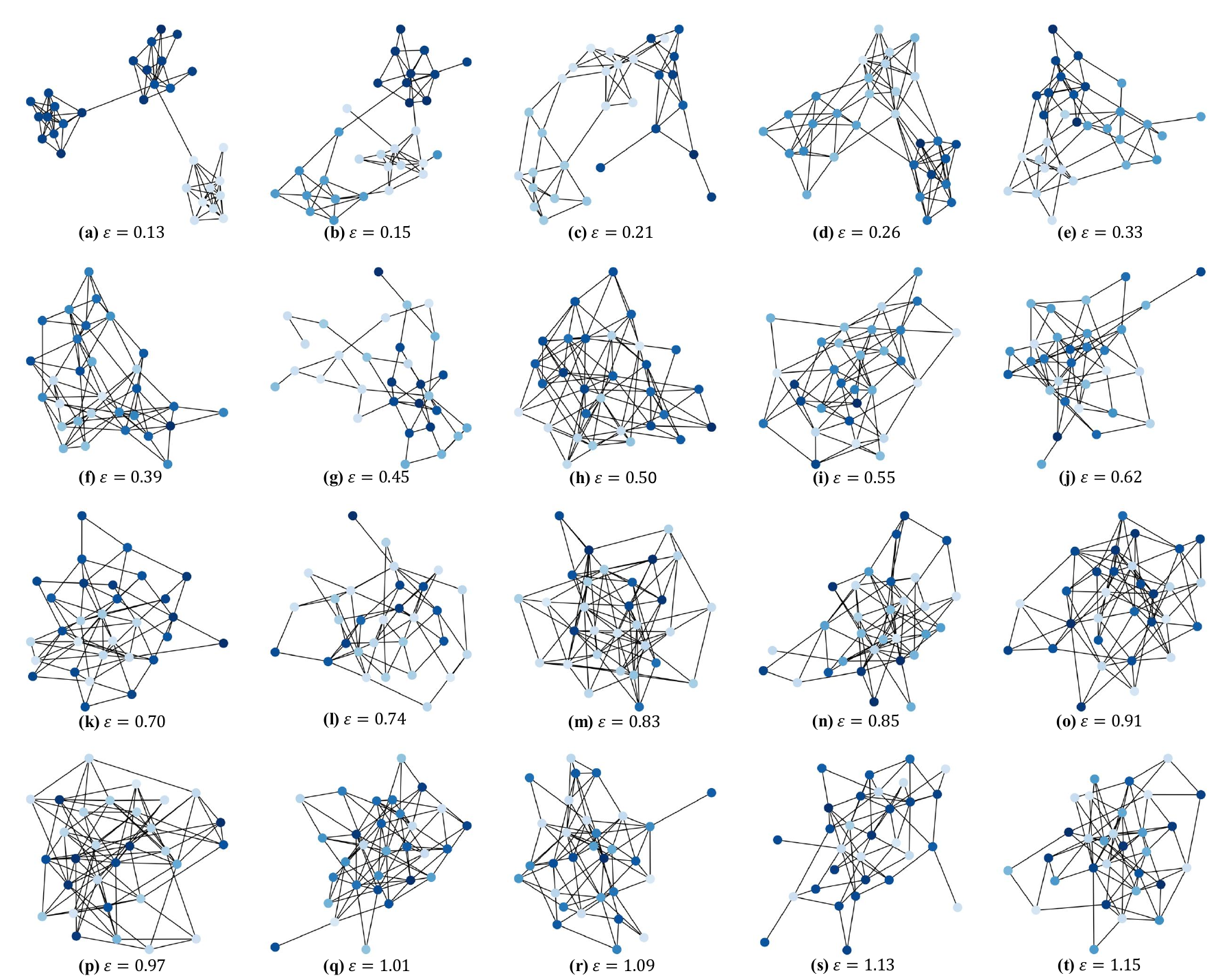}
\caption{Full visualizations of the universal structural basis with different Dirichlet energy $\mathcal{E}$ on the Mutagenicity dataset.}
\label{fig:vis_graph}
\end{figure*}

\input{figure/hyper_coeffiecent}
\input{figure/hyper_k}

\input{table/dd_ablation}
\input{table/frank_ablation}
\input{table/mutag_ablation}
\input{table/nci1_ablation}
\input{table/molhiv_ablation}
\input{table/dd_idx}
\input{table/dd_node}
\input{table/frank_idx}
\input{table/frank_node}

\input{table/mutag_idx}
\input{table/mutag_node}
\input{table/nci1_idx}
\input{table/nci1_node}
\input{table/molhiv_idx}
\input{table/molhiv_node}

·

%% file: figure/data_vis.tex
\begin{figure*}[h]
\includegraphics[width=1.0\linewidth]{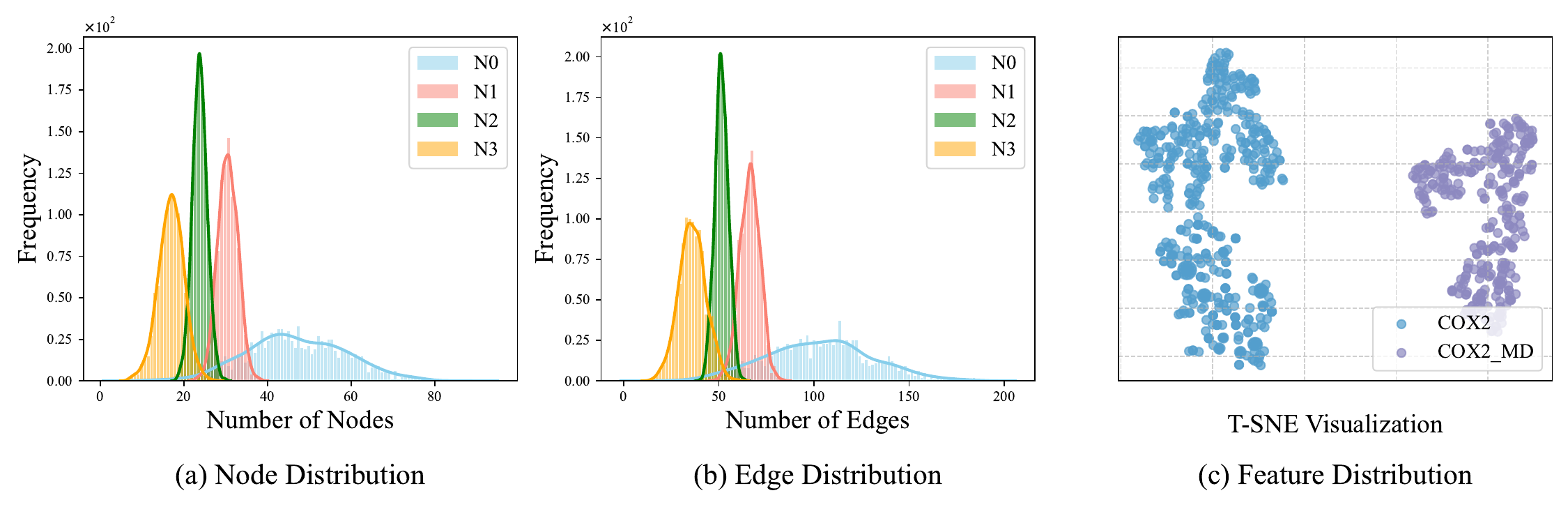}
    \caption{Visualization of domain shifts across different types. (a) Node distribution shift between sub-datasets of NCI1. (b) Edge distribution shift between sub-datasets of NCI1. (c) Feature distribution shift between COX2 and COX2\_MD datasets.}
    \label{fig:shift}
\end{figure*}

%% file: table/dataset.tex
\begin{table*}[h]
    \centering
    \caption{Statistics of the experimental datasets.}
    \begin{tabular}{lcccc}
        \toprule
        Datasets      & Graphs & Avg. Nodes & Avg. Edges & Classes \\
        \midrule
        DD & 1,178 & 284.32 & 715.66 & 2 \\
        NCI1  & 4,110   & 29.87     & 32.30     & 2       \\
        Mutagenicity  & 4,337   & 30.32      & 30.77      & 2       \\
        FRANKENSTEIN  & 4,337   & 16.9       & 17.88      & 2       \\
        ogbg-molhiv & 41,127 & 25.5 & 27.5 & 2 \\
        \midrule
        Spurious-Motif & 3,000 & 87.78 & 124.76 & 3 \\
        Spurious-Motif\_bias & 3,000 & 18.67 & 27.83 & 3 \\
        
        \midrule
        PROTEINS  & 1,113   & 39.1      & 72.8      & 2       \\
        COX2 & 467 & 41.22 & 43.45& 2 \\
        COX2\_MD & 303 & 26.28 &	335.12 & 2\\
        BZR & 405 & 35.75 & 38.36 & 2 \\
        BZR\_MD & 306 & 21.30 &	225.06 & 2 \\
        \bottomrule
    \end{tabular}
    \label{tab:dataset}
\end{table*}

%% file: figure/hyper_coeffiecent.tex
\begin{figure*}[t]

    \centering
    \captionsetup[subfigure]{font=scriptsize} 
    \begin{subfigure}[t]{0.24\textwidth}
        \centering
        \includegraphics[width=\linewidth]{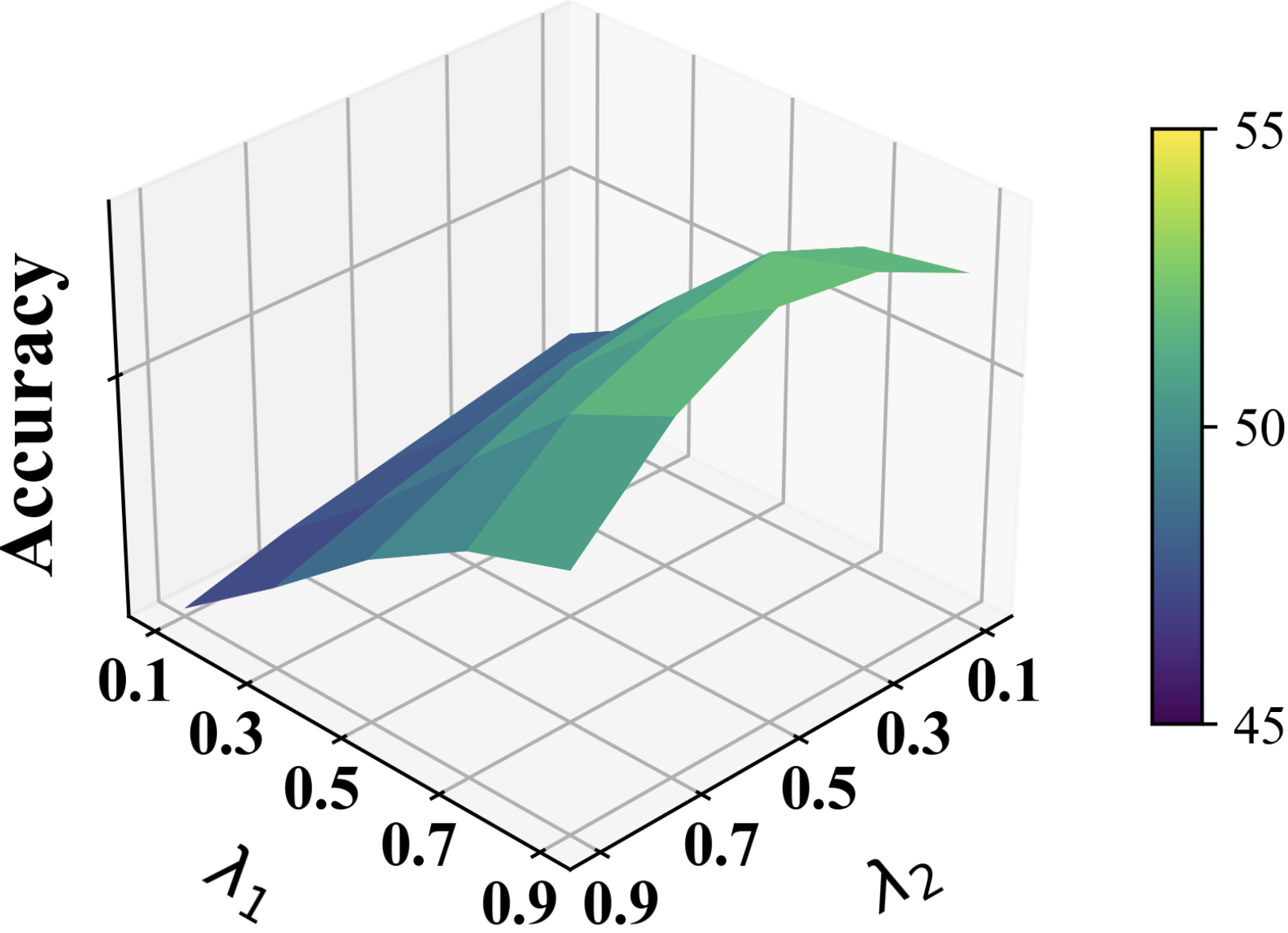}
        \caption{DD}
    \end{subfigure}
    \hfill
    \begin{subfigure}[t]{0.24\textwidth}
        \centering
        \includegraphics[width=\linewidth]{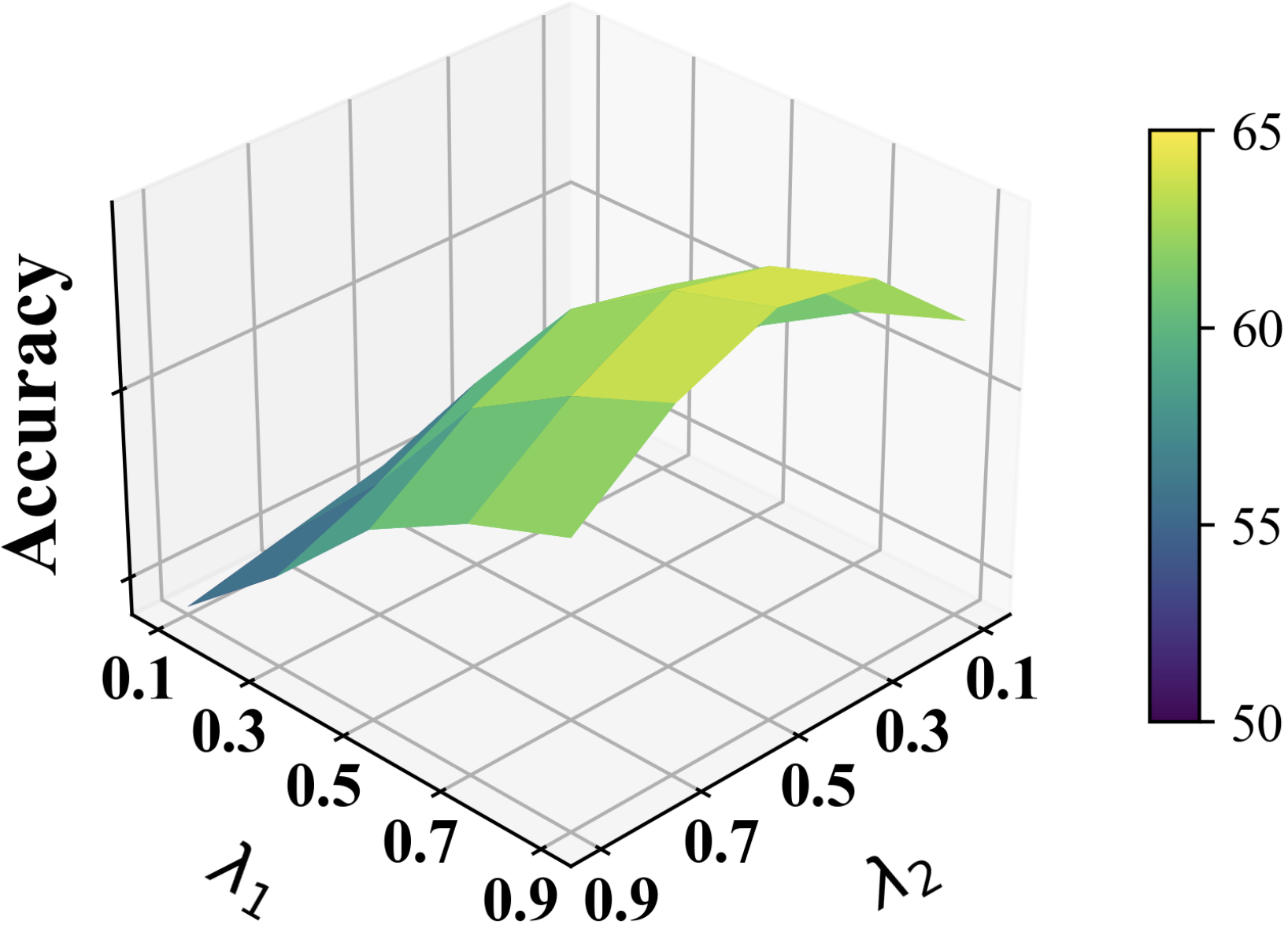}
        \caption{FRANKENSTEIN}
    \end{subfigure}
    \hfill
    \begin{subfigure}[t]{0.24\textwidth}
        \centering
        \includegraphics[width=\linewidth]{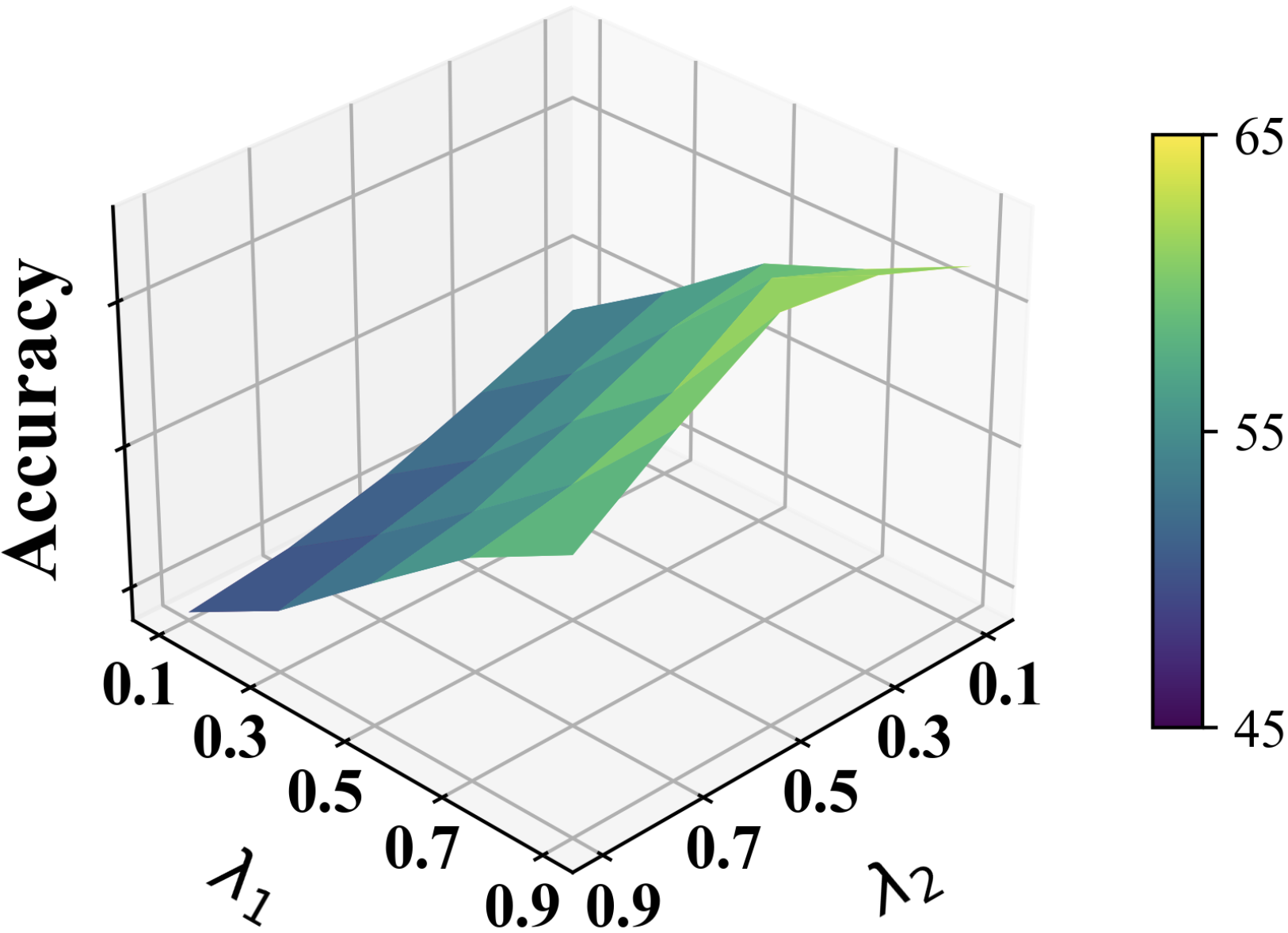}
        \caption{NCI1}
    \end{subfigure}
    \hfill
    \begin{subfigure}[t]{0.24\textwidth}
        \centering
        \includegraphics[width=\linewidth]{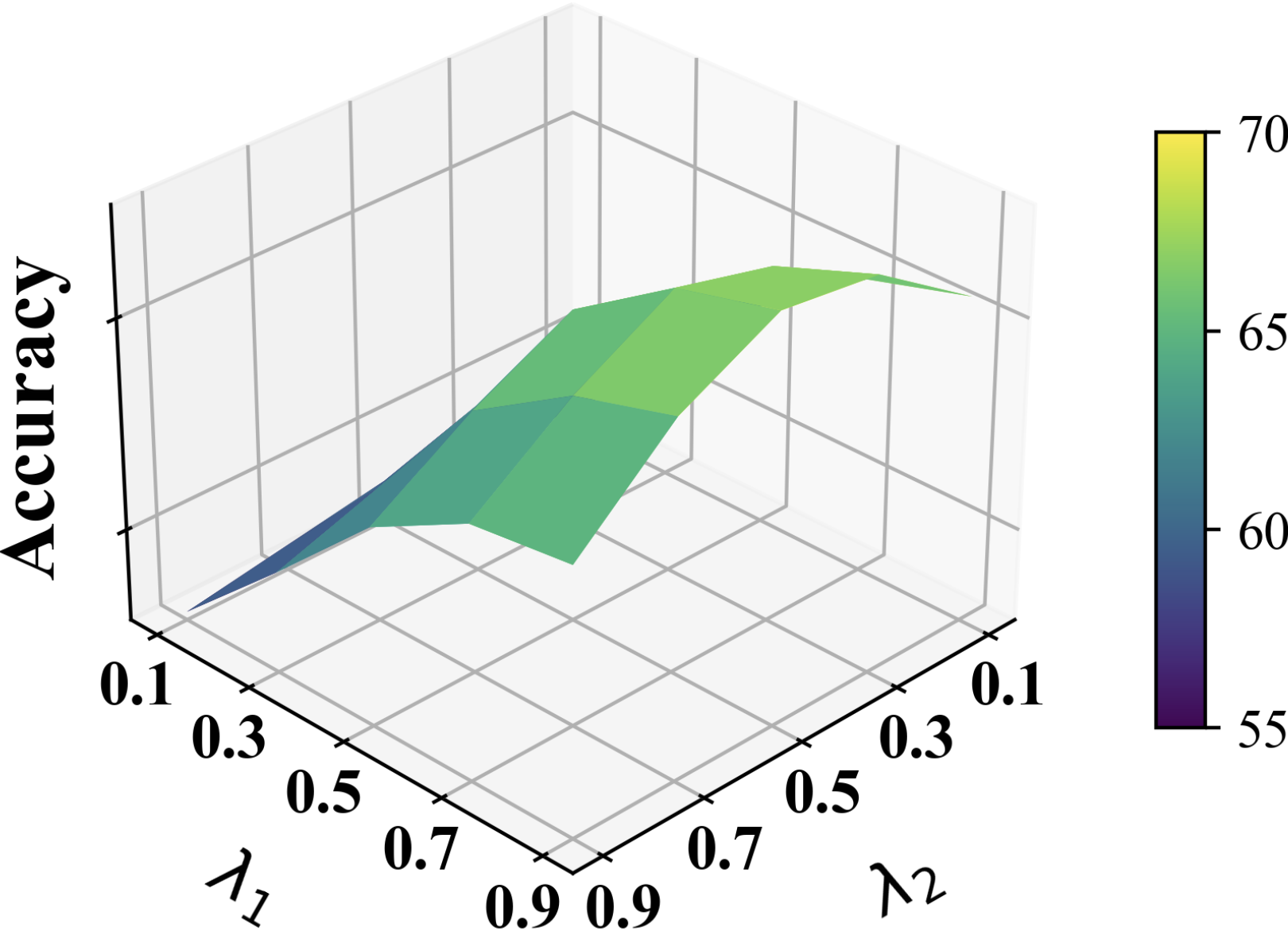}
        \caption{ogbg-molhiv}
    \end{subfigure}
    \caption{Hyperparameter sensitivity analysis of balance coefficient ($\lambda_\text{div}$, $\lambda_\text{span}$) on the DD, FRANKENSTEIN, NCI1 and ogbg-molhiv datasets.}
    \label{fig:hyper_coef}
\end{figure*}

%% file: figure/hyper_k.tex
\begin{figure*}[t]

    \centering
    \captionsetup[subfigure]{font=scriptsize} 
    \begin{subfigure}[t]{0.24\textwidth}
        \centering
        \includegraphics[width=\linewidth]{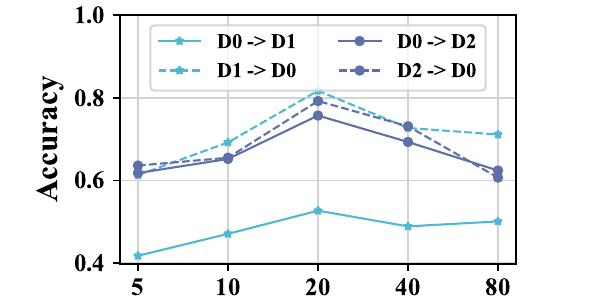}
        \caption{DD}
    \end{subfigure}
    \hfill
    \begin{subfigure}[t]{0.24\textwidth}
        \centering
        \includegraphics[width=\linewidth]{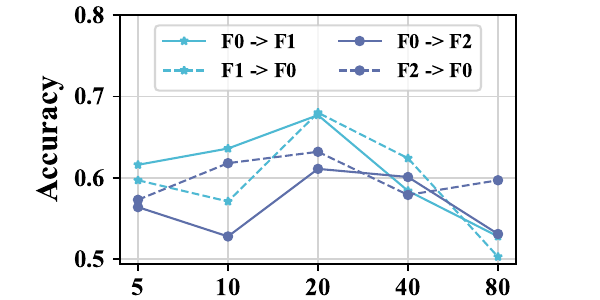}
        \caption{FRANKENSTEIN}
    \end{subfigure}
    \hfill
    \begin{subfigure}[t]{0.24\textwidth}
        \centering
        \includegraphics[width=\linewidth]{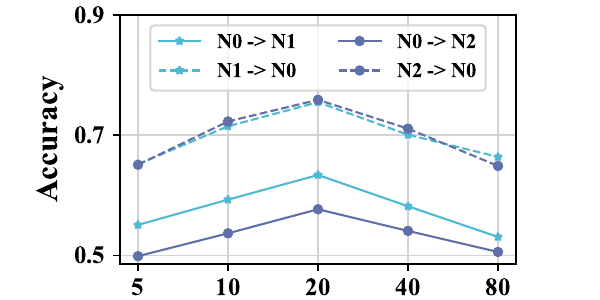}
        \caption{NCI1}
    \end{subfigure}
    \hfill
    \begin{subfigure}[t]{0.24\textwidth}
        \centering
        \includegraphics[width=\linewidth]{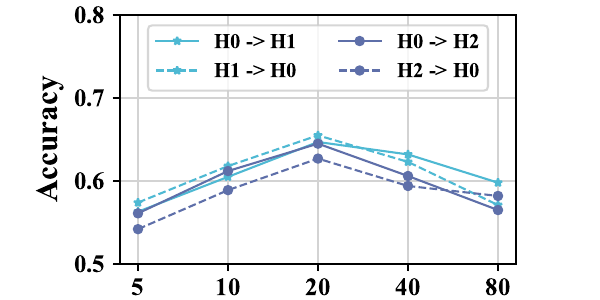}
        \caption{ogbg-molhiv}
    \end{subfigure}
    \caption{Hyperparameter sensitivity analysis of the number of synthetic bases
$K$ on the DD, FRANKENSTEIN, NCI1 and ogbg-molhiv datasets.}
    \label{fig:hyper_k}
\end{figure*}

%% file: table/dd_ablation.tex
\begin{table*}[ht]
\small
\centering
\caption{The results of ablation studies on the DD dataset (source $\rightarrow$ target). \textbf{Bold} results indicate the best performance.}
\resizebox{1.0\textwidth}{!}{
\begin{tabular}{l|c|c|c|c|c|c|c|c|c|c|c|c}
\toprule
\textbf{Methods}
& D0$\rightarrow$D1
& D1$\rightarrow$D0
& D0$\rightarrow$D2
& D2$\rightarrow$D0
& D0$\rightarrow$D3
& D3$\rightarrow$D0
& D1$\rightarrow$D2
& D2$\rightarrow$D1
& D1$\rightarrow$D3
& D3$\rightarrow$D1
& D2$\rightarrow$D3
& D3$\rightarrow$D2 \\
\midrule 
\method{} w/o SE  & 44.2 & 66.4 & 61.5 & 64.8 & 74.8 & 49.6 & 64.2 & 48.3 & 70.5 & 71.2 & 62.9 & 60.4 \\
\method{} w/o SP  & 51.6 & 77.5 & 71.8 & 75.4 & 85.2 & 56.8 & 75.8 & 56.5 & 82.1 & 82.5 & 74.3 & 71.2 \\
\method{} w/o DI  & 50.9 & 79.8 & 73.6 & 77.1 & 87.8 & 57.2 & 77.4 & 55.9 & 84.5 & 84.9 & 76.8 & 73.6 \\
\method{} w/o AD  & 47.1 & 74.2 & 68.2 & 71.5 & 81.5 & 53.7 & 71.9 & 52.4 & 78.6 & 78.4 & 70.1 & 67.5 \\ 
\midrule
\method{} & \textbf{52.7} & \textbf{81.7} & \textbf{75.7} & \textbf{79.2} & \textbf{89.6} & \textbf{59.4} & \textbf{79.2} & \textbf{58.7} & \textbf{86.4} & \textbf{86.7} & \textbf{78.5} & \textbf{75.1} \\ 
\bottomrule
\end{tabular}
}
\label{tab:ablation_dd}
\end{table*}

%% file: table/frank_ablation.tex
\begin{table*}[ht]
\small
\centering
\caption{The results of ablation studies on the FRANKENSTEIN dataset (source $\rightarrow$ target). \textbf{Bold} results indicate the best performance.}
\resizebox{1.0\textwidth}{!}{
\begin{tabular}{l|c|c|c|c|c|c|c|c|c|c|c|c}
\toprule
\textbf{Methods}
& F0$\rightarrow$F1 & F1$\rightarrow$F0
& F0$\rightarrow$F2 & F2$\rightarrow$F0
& F0$\rightarrow$F3 & F3$\rightarrow$F0
& F1$\rightarrow$F2 & F2$\rightarrow$F1
& F1$\rightarrow$F3 & F3$\rightarrow$F1
& F2$\rightarrow$F3 & F3$\rightarrow$F2 \\
\midrule 
\method{} w/o SE  & 58.5 & 58.8 & 52.5 & 54.6 & 53.8 & 56.5 & 55.4 & 59.6 & 53.5 & 52.1 & 53.1 & 56.8 \\
\method{} w/o SP  & 64.2 & 64.5 & 58.2 & 60.1 & 59.2 & 62.3 & 60.8 & 65.2 & 58.9 & 57.8 & 58.5 & 62.5 \\
\method{} w/o DI  & 66.5 & 66.8 & 60.1 & 62.1 & 61.2 & 64.5 & 63.0 & 67.1 & 61.0 & 59.8 & 60.6 & 64.4 \\
\method{} w/o AD  & 62.5 & 62.8 & 56.5 & 58.5 & 57.6 & 60.8 & 59.2 & 63.2 & 57.3 & 56.1 & 57.0 & 60.8 \\ 
\midrule
\method{} & \textbf{67.7} & \textbf{68.0} & \textbf{61.1} & \textbf{63.2} & \textbf{62.3} & \textbf{65.7} & \textbf{64.1} & {68.3} & \textbf{62.1} & \textbf{60.9} & \textbf{61.7} & \textbf{65.6} \\ 
\bottomrule
\end{tabular}
}
\label{tab:ablation_frankenstein}
\end{table*}

%% file: table/mutag_ablation.tex
\begin{table*}[ht]
\small
\centering
\caption{The results of ablation studies on the Mutagenicity dataset (source $\rightarrow$ target). \textbf{Bold} results indicate the best performance.}
\resizebox{1.0\textwidth}{!}{
\begin{tabular}{l|c|c|c|c|c|c|c|c|c|c|c|c}
\toprule
\textbf{Methods}
& M0$\rightarrow$M1 & M1$\rightarrow$M0
& M0$\rightarrow$M2 & M2$\rightarrow$M0
& M0$\rightarrow$M3 & M3$\rightarrow$M0
& M1$\rightarrow$M2 & M2$\rightarrow$M1
& M1$\rightarrow$M3 & M3$\rightarrow$M1
& M2$\rightarrow$M3 & M3$\rightarrow$M2 \\
\midrule 
\method{} w/o SE  & 58.7 & 59.3 & 57.2 & 60.7 & 54.4 & 56.8 & 58.1 & 59.2 & 56.8 & 57.3 & 56.6 & 58.6 \\
\method{} w/o SP  & 69.5 & 69.8 & 67.4 & 69.1 & 59.1 & 68.2 & 67.3 & 73.6 & 62.5 & 74.5 & 61.8 & 70.4 \\
\method{} w/o DI  & 71.9 & 72.5 & 70.8 & 72.4 & 63.5 & 71.5 & 70.6 & 76.9 & 65.9 & 77.8 & 65.4 & 73.9 \\
\method{} w/o AD  & 65.6 & 66.5 & 64.1 & 65.9 & 57.8 & 64.7 & 64.4 & 70.8 & 59.3 & 71.2 & 59.5 & 67.3 \\ 
\midrule
\method{} & \textbf{73.2} & \textbf{73.9} & \textbf{71.9} & \textbf{73.7} & \textbf{64.2} & \textbf{72.8} & \textbf{71.7} & \textbf{77.6} & \textbf{67.1} & \textbf{78.7} & \textbf{66.8} & \textbf{75.0} \\ 
\bottomrule
\end{tabular}
}
\label{tab:ablation_mutag}
\end{table*}

%% file: table/nci1_ablation.tex
\begin{table*}[ht]
\small
\centering
\caption{The results of ablation studies on the NCI1 dataset (source $\rightarrow$ target). \textbf{Bold} results indicate the best performance.}
\resizebox{1.0\textwidth}{!}{
\begin{tabular}{l|c|c|c|c|c|c|c|c|c|c|c|c}
\toprule
\textbf{Methods}
& N0$\rightarrow$N1
& N1$\rightarrow$N0
& N0$\rightarrow$N2
& N2$\rightarrow$N0
& N0$\rightarrow$N3
& N3$\rightarrow$N0
& N1$\rightarrow$N2
& N2$\rightarrow$N1
& N1$\rightarrow$N3
& N3$\rightarrow$N1
& N2$\rightarrow$N3
& N3$\rightarrow$N2 \\
\midrule 
\method{} w/o SE  & 50.5 & 63.0 & 54.0 & 62.5 & 55.5 & 60.0 & 53.0 & 52.0 & 58.0 & 51.0 & 57.5 & 52.5 \\
\method{} w/o SP  & 58.5 & 73.5 & 63.2 & 72.8 & 62.5 & 71.0 & 61.0 & 61.5 & 67.5 & 59.5 & 66.0 & 60.5 \\
\method{} w/o DI  & 61.8 & 74.8 & 66.8 & 75.1 & 68.1 & 73.2 & 64.2 & 64.0 & 69.8 & 62.8 & 69.2 & 63.9 \\
\method{} w/o AD  & 56.0 & 71.0 & 61.5 & 70.5 & 64.0 & 68.5 & 58.5 & 57.5 & 66.0 & 56.0 & 63.5 & 58.0 \\ 
\midrule
\method{} & \textbf{63.4} & \textbf{75.5} & \textbf{67.7} & \textbf{75.9} & \textbf{69.4} & \textbf{74.5} & \textbf{65.7} & \textbf{65.9} & \textbf{71.3} & \textbf{64.3} & \textbf{70.6} & \textbf{65.2} \\ 
\bottomrule
\end{tabular}
}
\label{tab:ablation_nci1}
\end{table*}

%% file: table/molhiv_ablation.tex
\begin{table*}[ht]
\small
\centering
\caption{The results of ablation studies on the ogbg-molhiv dataset (source $\rightarrow$ target). \textbf{Bold} results indicate the best performance.}
\resizebox{1.0\textwidth}{!}{
\begin{tabular}{l|c|c|c|c|c|c|c|c|c|c|c|c}
\toprule
\textbf{Methods}
& H0$\rightarrow$H1
& H1$\rightarrow$H0
& H0$\rightarrow$H2
& H2$\rightarrow$H0
& H0$\rightarrow$H3
& H3$\rightarrow$H0
& H1$\rightarrow$H2
& H2$\rightarrow$H1
& H1$\rightarrow$H3
& H3$\rightarrow$H1
& H2$\rightarrow$H3
& H3$\rightarrow$H2 \\
\midrule 
\method{} w/o SE  & 54.2 & 55.0 & 53.8 & 52.1 & 53.5 & 51.5 & 54.5 & 56.5 & 57.8 & 55.2 & 59.0 & 56.0 \\
\method{} w/o SP  & 61.2 & 62.1 & 62.5 & 59.5 & 60.5 & 58.2 & 64.0 & 64.5 & 66.8 & 62.8 & 68.5 & 66.1 \\
\method{} w/o DI  & 63.5 & 64.4 & 61.8 & 61.5 & 63.1 & 60.9 & 62.9 & 66.8 & 68.4 & 64.8 & 69.8 & 64.9 \\
\method{} w/o AD  & 59.5 & 60.2 & 58.9 & 58.2 & 58.0 & 58.5 & 60.1 & 62.3 & 64.5 & 60.5 & 66.2 & 61.8 \\ 
\midrule
\method{} & \textbf{64.7} & \textbf{65.5} & \textbf{64.5} & \textbf{62.7} & \textbf{64.3} & \textbf{62.1} & \textbf{65.7} & \textbf{67.8} & \textbf{69.2} & \textbf{65.9} & \textbf{70.4} & \textbf{67.6} \\ 
\bottomrule
\end{tabular}
}
\label{tab:ablation_molhiv}
\end{table*}

%% file: table/dd_idx.tex
\begin{table*}[t]
\small
\caption{The classification results (in \%) on the DD dataset under edge density domain shift (source $\rightarrow$ target). D0, D1, D2, and D3 denote the sub-datasets partitioned with edge density. \textbf{Bold} results indicate the best performance.}
\resizebox{1.0\textwidth}{!}{
\begin{tabular}{l|c|c|c|c|c|c|c|c|c|c|c|c}
\toprule
\textbf{Methods}
& D0$\rightarrow$D1
& D0$\rightarrow$D2
& D0$\rightarrow$D3
& D1$\rightarrow$D0
& D1$\rightarrow$D2
& D1$\rightarrow$D3
& D2$\rightarrow$D0
& D2$\rightarrow$D1
& D2$\rightarrow$D3
& D3$\rightarrow$D0
& D3$\rightarrow$D1
& D3$\rightarrow$D2 \\
\midrule
WL subtree
& 49.2 & 29.6 & 21.0
& 56.8 & 59.5 & 57.3
& 20.1 & 50.5 & 63.9
& 18.4 & 48.1 & 66.9 \\

GCN
& 48.9\scriptsize{$\pm$2.8} & 20.7\scriptsize{$\pm$2.0} & 15.1\scriptsize{$\pm$1.8}
& 59.0\scriptsize{$\pm$1.7} & 61.6\scriptsize{$\pm$1.9} & 68.1\scriptsize{$\pm$1.6}
& 27.3\scriptsize{$\pm$2.3} & 53.6\scriptsize{$\pm$1.5} & 64.9\scriptsize{$\pm$2.1}
& 26.9\scriptsize{$\pm$2.2} & 52.9\scriptsize{$\pm$2.6} & 69.7\scriptsize{$\pm$2.3} \\

GIN
& 48.8\scriptsize{$\pm$1.9} & 44.1\scriptsize{$\pm$1.8} & 57.0\scriptsize{$\pm$2.1}
& 24.7\scriptsize{$\pm$2.1} & 73.0\scriptsize{$\pm$1.8} & 63.2\scriptsize{$\pm$1.6}
& 22.4\scriptsize{$\pm$2.3} & 52.5\scriptsize{$\pm$2.3} & 70.3\scriptsize{$\pm$2.6}
& 18.4\scriptsize{$\pm$2.0} & 53.6\scriptsize{$\pm$1.5} & 69.4\scriptsize{$\pm$1.8} \\

GMT
& 49.1\scriptsize{$\pm$1.9} & 31.8\scriptsize{$\pm$1.8} & 52.5\scriptsize{$\pm$1.5}
& 32.9\scriptsize{$\pm$2.2} & 75.4\scriptsize{$\pm$1.9} & 74.1\scriptsize{$\pm$2.6}
& 27.3\scriptsize{$\pm$2.3} & 53.2\scriptsize{$\pm$2.1} & 70.9\scriptsize{$\pm$1.8}
& 27.6\scriptsize{$\pm$1.8} & 57.9\scriptsize{$\pm$2.4} & 71.1\scriptsize{$\pm$2.7} \\

CIN
& 50.4\scriptsize{$\pm$1.8} & 21.2\scriptsize{$\pm$2.1} & 43.0\scriptsize{$\pm$2.1}
& 18.4\scriptsize{$\pm$2.0} & 53.4\scriptsize{$\pm$1.7} & 62.3\scriptsize{$\pm$1.6}
& 36.8\scriptsize{$\pm$1.8} & 56.5\scriptsize{$\pm$1.5} & 75.0\scriptsize{$\pm$2.1}
& 22.9\scriptsize{$\pm$1.9} & 53.3\scriptsize{$\pm$1.9} & 69.3\scriptsize{$\pm$2.0} \\
\bottomrule
\end{tabular}}
\label{tab:dd_idx}
\end{table*}

%% file: table/dd_node.tex
\begin{table*}[t]
\small
\caption{The classification results (in \%) on the DD dataset under node density domain shift (source $\rightarrow$ target). D0, D1, D2, and D3 denote the sub-datasets partitioned with node density. \textbf{Bold} results indicate the best performance.}
\resizebox{1.0\textwidth}{!}{
\begin{tabular}{l|c|c|c|c|c|c|c|c|c|c|c|c}
\toprule
\textbf{Methods}
& D0$\rightarrow$D1
& D0$\rightarrow$D2
& D0$\rightarrow$D3
& D1$\rightarrow$D0
& D1$\rightarrow$D2
& D1$\rightarrow$D3
& D2$\rightarrow$D0
& D2$\rightarrow$D1
& D2$\rightarrow$D3
& D3$\rightarrow$D0
& D3$\rightarrow$D1
& D3$\rightarrow$D2 \\
\midrule
WL subtree
& 49.2 & 29.6 & 21.0
& 56.8 & 59.5 & 57.3
& 20.1 & 50.5 & 63.9
& 18.4 & 48.1 & 66.9 \\

GCN
& 48.9\scriptsize{$\pm$2.8} & 20.7\scriptsize{$\pm$2.0} & 15.1\scriptsize{$\pm$1.8}
& 59.0\scriptsize{$\pm$1.7} & 61.6\scriptsize{$\pm$1.9} & 68.1\scriptsize{$\pm$1.6}
& 27.3\scriptsize{$\pm$2.3} & 53.6\scriptsize{$\pm$1.5} & 64.9\scriptsize{$\pm$2.1}
& 26.9\scriptsize{$\pm$2.2} & 52.9\scriptsize{$\pm$2.6} & 69.7\scriptsize{$\pm$2.3} \\

GIN
& 48.8\scriptsize{$\pm$1.9} & 44.1\scriptsize{$\pm$1.8} & 57.0\scriptsize{$\pm$2.1}
& 24.7\scriptsize{$\pm$2.1} & 73.0\scriptsize{$\pm$1.8} & 63.2\scriptsize{$\pm$1.6}
& 22.4\scriptsize{$\pm$2.3} & 52.5\scriptsize{$\pm$2.3} & 70.3\scriptsize{$\pm$2.6}
& 18.4\scriptsize{$\pm$2.0} & 53.6\scriptsize{$\pm$1.5} & 69.4\scriptsize{$\pm$1.8} \\

GMT
& 49.1\scriptsize{$\pm$1.9} & 31.8\scriptsize{$\pm$1.8} & 52.5\scriptsize{$\pm$1.5}
& 32.9\scriptsize{$\pm$2.2} & 75.4\scriptsize{$\pm$1.9} & 74.1\scriptsize{$\pm$2.6}
& 27.3\scriptsize{$\pm$2.3} & 53.2\scriptsize{$\pm$2.1} & 70.9\scriptsize{$\pm$1.8}
& 27.6\scriptsize{$\pm$1.8} & 57.9\scriptsize{$\pm$2.4} & 71.1\scriptsize{$\pm$2.7} \\

CIN
& 50.4\scriptsize{$\pm$1.8} & 21.2\scriptsize{$\pm$2.1} & 43.0\scriptsize{$\pm$2.1}
& 18.4\scriptsize{$\pm$2.0} & 53.4\scriptsize{$\pm$1.7} & 62.3\scriptsize{$\pm$1.6}
& 36.8\scriptsize{$\pm$1.8} & 56.5\scriptsize{$\pm$1.5} & 75.0\scriptsize{$\pm$2.1}
& 22.9\scriptsize{$\pm$1.9} & 53.3\scriptsize{$\pm$1.9} & 69.3\scriptsize{$\pm$2.0} \\
\bottomrule
\end{tabular}}
\label{tab:dd_node}
\end{table*}

%% file: table/frank_idx.tex
\begin{table*}[t]
\small
\caption{The classification results (in \%) on the FRANKENSTEIN dataset under edge density domain shift (source $\rightarrow$ target). F0, F1, F2, and F3 denote the sub-datasets partitioned with edge density. \textbf{Bold} results indicate the best performance.}
\resizebox{1.0\textwidth}{!}{
\begin{tabular}{l|c|c|c|c|c|c|c|c|c|c|c|c}
\toprule
\textbf{Methods}
& F0$\rightarrow$F1
& F0$\rightarrow$F2
& F0$\rightarrow$F3
& F1$\rightarrow$F0
& F1$\rightarrow$F2
& F1$\rightarrow$F3
& F2$\rightarrow$F0
& F2$\rightarrow$F1
& F2$\rightarrow$F3
& F3$\rightarrow$F0
& F3$\rightarrow$F1
& F3$\rightarrow$F2 \\
\midrule
WL subtree
& 44.6 & 47.0 & 58.2
& 54.8 & 50.8 & 45.7
& 56.1 & 58.1 & 45.3
& 41.8 & 47.9 & 42.7 \\

GCN
& 33.5\scriptsize{$\pm$2.1} & 44.6\scriptsize{$\pm$1.3} & 57.4\scriptsize{$\pm$2.0}
& 40.9\scriptsize{$\pm$1.5} & 44.6\scriptsize{$\pm$2.0} & 56.4\scriptsize{$\pm$1.1}
& 40.1\scriptsize{$\pm$1.6} & 33.5\scriptsize{$\pm$1.2} & 57.4\scriptsize{$\pm$2.0}
& 43.1\scriptsize{$\pm$1.3} & 36.6\scriptsize{$\pm$1.9} & 47.8\scriptsize{$\pm$1.7} \\

GIN
& 40.1\scriptsize{$\pm$1.2} & 52.2\scriptsize{$\pm$2.2} & 58.3\scriptsize{$\pm$1.5}
& 52.4\scriptsize{$\pm$1.8} & 53.2\scriptsize{$\pm$2.3} & 52.1\scriptsize{$\pm$2.2}
& 52.7\scriptsize{$\pm$1.8} & 47.1\scriptsize{$\pm$1.6} & 49.9\scriptsize{$\pm$1.8}
& 42.8\scriptsize{$\pm$2.0} & 39.1\scriptsize{$\pm$2.0} & 43.5\scriptsize{$\pm$1.9} \\

GMT
& 51.5\scriptsize{$\pm$1.0} & 55.4\scriptsize{$\pm$1.6} & 39.6\scriptsize{$\pm$1.6}
& 57.0\scriptsize{$\pm$1.3} & 55.4\scriptsize{$\pm$2.1} & 39.6\scriptsize{$\pm$1.9}
& 59.0\scriptsize{$\pm$1.2} & 66.5\scriptsize{$\pm$1.6} & 39.6\scriptsize{$\pm$1.4}
& 44.1\scriptsize{$\pm$2.0} & 36.5\scriptsize{$\pm$1.5} & 42.8\scriptsize{$\pm$1.4} \\

CIN
& 54.2\scriptsize{$\pm$1.1} & 47.6\scriptsize{$\pm$1.8} & 56.2\scriptsize{$\pm$1.3}
& 52.2\scriptsize{$\pm$1.9} & 46.8\scriptsize{$\pm$1.3} & 49.1\scriptsize{$\pm$1.5}
& 52.5\scriptsize{$\pm$2.2} & 59.9\scriptsize{$\pm$1.2} & 57.8\scriptsize{$\pm$1.4}
& 40.0\scriptsize{$\pm$2.0} & 33.5\scriptsize{$\pm$2.0} & 45.9\scriptsize{$\pm$2.0} \\
\bottomrule
\end{tabular}}
\label{tab:frank_idx}
\end{table*}

%% file: table/frank_node.tex
\begin{table*}[t]
\small
\caption{The classification results (in \%) on the FRANKENSTEIN dataset under node density domain shift (source $\rightarrow$ target). F0, F1, F2, and F3 denote the sub-datasets partitioned with node density. \textbf{Bold} results indicate the best performance.}
\resizebox{1.0\textwidth}{!}{
\begin{tabular}{l|c|c|c|c|c|c|c|c|c|c|c|c}
\toprule
\textbf{Methods}
& F0$\rightarrow$F1
& F0$\rightarrow$F2
& F0$\rightarrow$F3
& F1$\rightarrow$F0
& F1$\rightarrow$F2
& F1$\rightarrow$F3
& F2$\rightarrow$F0
& F2$\rightarrow$F1
& F2$\rightarrow$F3
& F3$\rightarrow$F0
& F3$\rightarrow$F1
& F3$\rightarrow$F2 \\
\midrule
WL subtree
& 54.6 & 47.8 & 49.6
& 40.1 & 49.7 & 41.8
& 62.0 & 51.2 & 47.3
& 56.0 & 42.9 & 45.1 \\

GCN
& 32.7\scriptsize{$\pm$2.0} & 42.2\scriptsize{$\pm$1.7} & 54.5\scriptsize{$\pm$1.8}
& 43.2\scriptsize{$\pm$2.1} & 42.2\scriptsize{$\pm$2.0} & 56.5\scriptsize{$\pm$1.2}
& 43.2\scriptsize{$\pm$1.5} & 32.7\scriptsize{$\pm$1.5} & 55.5\scriptsize{$\pm$1.8}
& 43.2\scriptsize{$\pm$1.9} & 32.7\scriptsize{$\pm$1.8} & 42.2\scriptsize{$\pm$2.0} \\

GIN
& 39.7\scriptsize{$\pm$1.8} & 54.7\scriptsize{$\pm$1.3} & 58.2\scriptsize{$\pm$1.9}
& 51.4\scriptsize{$\pm$1.7} & 54.7\scriptsize{$\pm$2.3} & 52.1\scriptsize{$\pm$1.2}
& 51.4\scriptsize{$\pm$1.7} & 47.1\scriptsize{$\pm$1.6} & 50.8\scriptsize{$\pm$1.7}
& 41.8\scriptsize{$\pm$2.0} & 34.3\scriptsize{$\pm$1.5} & 43.8\scriptsize{$\pm$2.0} \\

GMT
& 47.3\scriptsize{$\pm$2.2} & 57.8\scriptsize{$\pm$1.4} & 39.5\scriptsize{$\pm$2.3}
& 56.8\scriptsize{$\pm$1.9} & 57.8\scriptsize{$\pm$2.0} & 39.5\scriptsize{$\pm$1.6}
& 56.8\scriptsize{$\pm$1.2} & 58.3\scriptsize{$\pm$1.5} & 39.5\scriptsize{$\pm$2.0}
& 44.1\scriptsize{$\pm$1.2} & 39.3\scriptsize{$\pm$1.4} & 44.1\scriptsize{$\pm$1.6} \\

CIN
& 54.3\scriptsize{$\pm$1.8} & 45.0\scriptsize{$\pm$1.7} & 51.9\scriptsize{$\pm$2.1}
& 51.6\scriptsize{$\pm$1.9} & 45.3\scriptsize{$\pm$1.3} & 48.8\scriptsize{$\pm$1.6}
& 52.5\scriptsize{$\pm$1.8} & 60.4\scriptsize{$\pm$1.8} & 57.1\scriptsize{$\pm$1.2}
& 43.2\scriptsize{$\pm$2.0} & 32.7\scriptsize{$\pm$2.0} & 42.2\scriptsize{$\pm$1.8} \\
\bottomrule
\end{tabular}}
\label{tab:frank_node}
\end{table*}

%% file: table/mutag_idx.tex
\begin{table*}[t]
\small
\caption{The classification results (in \%) on the Mutagenicity dataset under edge density domain shift (source $\rightarrow$ target). M0, M1, M2, and M3 denote the sub-datasets partitioned with edge density. \textbf{Bold} results indicate the best performance.}
\resizebox{1.0\textwidth}{!}{
\begin{tabular}{l|c|c|c|c|c|c|c|c|c|c|c|c}
\toprule
\textbf{Methods}
& M0$\rightarrow$M1
& M0$\rightarrow$M2
& M0$\rightarrow$M3
& M1$\rightarrow$M0
& M1$\rightarrow$M2
& M1$\rightarrow$M3
& M2$\rightarrow$M0
& M2$\rightarrow$M1
& M2$\rightarrow$M3
& M3$\rightarrow$M0
& M3$\rightarrow$M1
& M3$\rightarrow$M2 \\
\midrule
WL subtree
& 34.4 & 47.6 & 52.7
& 39.6 & 59.8 & 53.0
& 43.6 & 40.9 & 52.6
& 46.4 & 65.1 & 42.4 \\

GCN
& 66.3\scriptsize{$\pm$1.7} & 63.6\scriptsize{$\pm$1.4} & 56.0\scriptsize{$\pm$1.4}
& 65.1\scriptsize{$\pm$2.2} & 66.1\scriptsize{$\pm$2.0} & 48.0\scriptsize{$\pm$2.3}
& 67.6\scriptsize{$\pm$1.6} & 71.8\scriptsize{$\pm$1.8} & 51.6\scriptsize{$\pm$1.8}
& 49.6\scriptsize{$\pm$1.9} & 47.8\scriptsize{$\pm$1.7} & 56.6\scriptsize{$\pm$2.2} \\

GIN
& 67.1\scriptsize{$\pm$1.7} & 54.2\scriptsize{$\pm$2.6} & 55.4\scriptsize{$\pm$1.9}
& 56.1\scriptsize{$\pm$1.5} & 67.9\scriptsize{$\pm$2.1} & 53.4\scriptsize{$\pm$1.6}
& 65.6\scriptsize{$\pm$1.6} & 69.4\scriptsize{$\pm$2.0} & 50.4\scriptsize{$\pm$1.1}
& 50.8\scriptsize{$\pm$2.1} & 66.7\scriptsize{$\pm$2.0} & 55.4\scriptsize{$\pm$1.2} \\

GMT
& 67.9\scriptsize{$\pm$1.3} & 61.5\scriptsize{$\pm$1.8} & 58.2\scriptsize{$\pm$2.4}
& 63.9\scriptsize{$\pm$1.8} & 66.7\scriptsize{$\pm$1.9} & 58.8\scriptsize{$\pm$1.5}
& 68.5\scriptsize{$\pm$1.5} & 76.1\scriptsize{$\pm$1.7} & 59.0\scriptsize{$\pm$1.4}
& 54.3\scriptsize{$\pm$1.8} & 44.5\scriptsize{$\pm$1.5} & 63.7\scriptsize{$\pm$1.7} \\

CIN
& 66.3\scriptsize{$\pm$1.8} & 60.8\scriptsize{$\pm$1.7} & 55.8\scriptsize{$\pm$2.4}
& 56.7\scriptsize{$\pm$2.5} & 63.4\scriptsize{$\pm$2.3} & 54.3\scriptsize{$\pm$1.2}
& 67.4\scriptsize{$\pm$1.7} & 69.3\scriptsize{$\pm$1.2} & 49.0\scriptsize{$\pm$1.3}
& 52.3\scriptsize{$\pm$1.8} & 60.7\scriptsize{$\pm$1.4} & 54.4\scriptsize{$\pm$1.7} \\
\bottomrule
\end{tabular}}
\label{tab:mutag_idx}
\end{table*}

%% file: table/mutag_node.tex
\begin{table*}[t]
\small
\caption{The classification results (in \%) on the Mutagenicity dataset under node density domain shift (source $\rightarrow$ target). M0, M1, M2, and M3 denote the sub-datasets partitioned with node density. \textbf{Bold} results indicate the best performance.}
\resizebox{1.0\textwidth}{!}{
\begin{tabular}{l|c|c|c|c|c|c|c|c|c|c|c|c}
\toprule
\textbf{Methods}
& M0$\rightarrow$M1
& M0$\rightarrow$M2
& M0$\rightarrow$M3
& M1$\rightarrow$M0
& M1$\rightarrow$M2
& M1$\rightarrow$M3
& M2$\rightarrow$M0
& M2$\rightarrow$M1
& M2$\rightarrow$M3
& M3$\rightarrow$M0
& M3$\rightarrow$M1
& M3$\rightarrow$M2 \\
\midrule
WL subtree
& 34.3 & 40.4 & 52.7
& 46.0 & 51.6 & 45.6
& 50.4 & 48.3 & 53.5
& 52.1 & 40.8 & 52.4 \\

GCN
& 64.1\scriptsize{$\pm$1.4} & 65.5\scriptsize{$\pm$2.0} & 56.9\scriptsize{$\pm$2.1}
& 62.6\scriptsize{$\pm$2.3} & 66.1\scriptsize{$\pm$1.6} & 47.7\scriptsize{$\pm$1.9}
& 62.3\scriptsize{$\pm$1.1} & 70.8\scriptsize{$\pm$2.1} & 52.1\scriptsize{$\pm$1.7}
& 53.7\scriptsize{$\pm$1.8} & 49.2\scriptsize{$\pm$2.1} & 54.4\scriptsize{$\pm$2.1} \\

GIN
& 66.5\scriptsize{$\pm$2.1} & 52.0\scriptsize{$\pm$1.7} & 53.7\scriptsize{$\pm$1.7}
& 53.2\scriptsize{$\pm$2.3} & 71.5\scriptsize{$\pm$2.5} & 52.2\scriptsize{$\pm$2.3}
& 62.6\scriptsize{$\pm$2.3} & 70.0\scriptsize{$\pm$2.1} & 50.9\scriptsize{$\pm$1.4}
& 48.6\scriptsize{$\pm$2.3} & 64.3\scriptsize{$\pm$2.1} & 52.4\scriptsize{$\pm$1.5} \\

GMT
& 65.7\scriptsize{$\pm$1.8} & 62.1\scriptsize{$\pm$2.1} & 59.0\scriptsize{$\pm$2.0}
& 67.1\scriptsize{$\pm$1.9} & 69.9\scriptsize{$\pm$1.4} & 60.4\scriptsize{$\pm$1.3}
& 66.4\scriptsize{$\pm$1.2} & 68.4\scriptsize{$\pm$2.3} & 59.1\scriptsize{$\pm$1.3}
& 58.0\scriptsize{$\pm$2.7} & 48.2\scriptsize{$\pm$2.0} & 60.6\scriptsize{$\pm$1.7} \\

CIN
& 65.1\scriptsize{$\pm$1.7} & 66.0\scriptsize{$\pm$1.7} & 55.2\scriptsize{$\pm$1.5}
& 52.9\scriptsize{$\pm$2.3} & 66.0\scriptsize{$\pm$1.6} & 54.9\scriptsize{$\pm$1.8}
& 69.8\scriptsize{$\pm$2.0} & 69.9\scriptsize{$\pm$2.1} & 46.7\scriptsize{$\pm$2.0}
& 55.1\scriptsize{$\pm$1.3} & 57.8\scriptsize{$\pm$1.5} & 51.8\scriptsize{$\pm$1.1} \\
\bottomrule
\end{tabular}}
\label{tab:mutag_node}
\end{table*}

%% file: table/nci1_idx.tex
\begin{table*}[t]
\small
\caption{The classification results (in \%) on the NCI1 dataset under edge density domain shift (source $\rightarrow$ target). N0, N1, N2, and N3 denote the sub-datasets partitioned with edge density. \textbf{Bold} results indicate the best performance.}
\resizebox{1.0\textwidth}{!}{
\begin{tabular}{l|c|c|c|c|c|c|c|c|c|c|c|c}
\toprule
\textbf{Methods}
& N0$\rightarrow$N1
& N0$\rightarrow$N2
& N0$\rightarrow$N3
& N1$\rightarrow$N0
& N1$\rightarrow$N2
& N1$\rightarrow$N3
& N2$\rightarrow$N0
& N2$\rightarrow$N1
& N2$\rightarrow$N3
& N3$\rightarrow$N0
& N3$\rightarrow$N1
& N3$\rightarrow$N2 \\
\midrule
WL subtree
& 50.1 & 50.9 & 51.8
& 59.9 & 56.7 & 52.5
& 51.0 & 57.0 & 54.4
& 54.5 & 57.6 & 46.9 \\

GCN
& 53.7\scriptsize{$\pm$2.1} & 50.6\scriptsize{$\pm$1.9} & 35.5\scriptsize{$\pm$2.4}
& 68.2\scriptsize{$\pm$1.9} & 55.9\scriptsize{$\pm$1.7} & 55.7\scriptsize{$\pm$1.5}
& 45.8\scriptsize{$\pm$1.7} & 56.8\scriptsize{$\pm$1.9} & 60.4\scriptsize{$\pm$1.8}
& 28.2\scriptsize{$\pm$2.2} & 53.2\scriptsize{$\pm$2.1} & 54.8\scriptsize{$\pm$0.9} \\

GIN
& 51.4\scriptsize{$\pm$1.7} & 50.4\scriptsize{$\pm$1.6} & 47.0\scriptsize{$\pm$2.1}
& 47.5\scriptsize{$\pm$2.1} & 52.4\scriptsize{$\pm$1.9} & 64.3\scriptsize{$\pm$1.8}
& 32.7\scriptsize{$\pm$1.6} & 55.7\scriptsize{$\pm$2.0} & 60.3\scriptsize{$\pm$1.3}
& 29.4\scriptsize{$\pm$2.3} & 51.1\scriptsize{$\pm$1.2} & 53.5\scriptsize{$\pm$1.6} \\

GMT
& 53.0\scriptsize{$\pm$1.9} & 49.0\scriptsize{$\pm$1.9} & 39.7\scriptsize{$\pm$1.5}
& 58.3\scriptsize{$\pm$1.7} & 55.8\scriptsize{$\pm$1.6} & 51.3\scriptsize{$\pm$1.0}
& 34.3\scriptsize{$\pm$1.6} & 60.4\scriptsize{$\pm$2.1} & 60.0\scriptsize{$\pm$2.6}
& 27.5\scriptsize{$\pm$2.4} & 51.5\scriptsize{$\pm$1.6} & 54.5\scriptsize{$\pm$1.4} \\

CIN
& 54.5\scriptsize{$\pm$2.1} & 53.0\scriptsize{$\pm$1.3} & 44.9\scriptsize{$\pm$1.4}
& 36.8\scriptsize{$\pm$1.9} & 58.4\scriptsize{$\pm$1.6} & 63.9\scriptsize{$\pm$1.6}
& 32.7\scriptsize{$\pm$2.1} & 54.2\scriptsize{$\pm$2.3} & 64.7\scriptsize{$\pm$1.1}
& 27.4\scriptsize{$\pm$1.1} & 51.2\scriptsize{$\pm$1.4} & 53.7\scriptsize{$\pm$2.1} \\
\bottomrule
\end{tabular}}
\label{tab:nci_idx}
\end{table*}

%% file: table/nci1_node.tex
\begin{table*}[t]
\small
\caption{The classification results (in \%) on the NCI1 dataset under node density domain shift (source $\rightarrow$ target). N0, N1, N2, and N3 denote the sub-datasets partitioned with node density. \textbf{Bold} results indicate the best performance.}
\resizebox{1.0\textwidth}{!}{
\begin{tabular}{l|c|c|c|c|c|c|c|c|c|c|c|c}
\toprule
\textbf{Methods}
& N0$\rightarrow$N1
& N0$\rightarrow$N2
& N0$\rightarrow$N3
& N1$\rightarrow$N0
& N1$\rightarrow$N2
& N1$\rightarrow$N3
& N2$\rightarrow$N0
& N2$\rightarrow$N1
& N2$\rightarrow$N3
& N3$\rightarrow$N0
& N3$\rightarrow$N1
& N3$\rightarrow$N2 \\
\midrule
WL subtree
& 54.9 & 51.4 & 44.4
& 60.0 & 51.9 & 58.4
& 51.4 & 53.4 & 61.2
& 63.1 & 57.4 & 50.9 \\

GCN
& 54.0\scriptsize{$\pm$2.0} & 47.3\scriptsize{$\pm$2.2} & 38.4\scriptsize{$\pm$2.3}
& 66.5\scriptsize{$\pm$1.9} & 59.0\scriptsize{$\pm$1.7} & 56.5\scriptsize{$\pm$1.2}
& 37.4\scriptsize{$\pm$1.8} & 56.9\scriptsize{$\pm$2.0} & 60.5\scriptsize{$\pm$1.9}
& 27.6\scriptsize{$\pm$1.5} & 51.0\scriptsize{$\pm$1.7} & 57.4\scriptsize{$\pm$1.6} \\

GIN
& 57.0\scriptsize{$\pm$1.9} & 45.7\scriptsize{$\pm$1.8} & 43.6\scriptsize{$\pm$2.1}
& 38.7\scriptsize{$\pm$2.3} & 54.7\scriptsize{$\pm$1.7} & 59.0\scriptsize{$\pm$1.8}
& 34.4\scriptsize{$\pm$1.9} & 52.8\scriptsize{$\pm$2.1} & 59.4\scriptsize{$\pm$2.0}
& 27.0\scriptsize{$\pm$2.0} & 49.1\scriptsize{$\pm$2.5} & 57.4\scriptsize{$\pm$1.3} \\

GMT
& 53.0\scriptsize{$\pm$1.9} & 48.9\scriptsize{$\pm$2.2} & 41.5\scriptsize{$\pm$1.7}
& 63.7\scriptsize{$\pm$1.6} & 56.4\scriptsize{$\pm$1.9} & 54.1\scriptsize{$\pm$1.9}
& 32.8\scriptsize{$\pm$2.0} & 57.2\scriptsize{$\pm$1.7} & 61.6\scriptsize{$\pm$2.1}
& 27.8\scriptsize{$\pm$1.6} & 49.8\scriptsize{$\pm$1.7} & 58.5\scriptsize{$\pm$1.1} \\

CIN
& 54.5\scriptsize{$\pm$1.9} & 53.2\scriptsize{$\pm$1.8} & 42.1\scriptsize{$\pm$1.8}
& 37.6\scriptsize{$\pm$1.9} & 57.2\scriptsize{$\pm$1.6} & 59.3\scriptsize{$\pm$1.5}
& 36.2\scriptsize{$\pm$1.4} & 51.6\scriptsize{$\pm$1.9} & 62.0\scriptsize{$\pm$1.4}
& 27.5\scriptsize{$\pm$2.0} & 50.0\scriptsize{$\pm$1.4} & 57.4\scriptsize{$\pm$2.3} \\
\bottomrule
\end{tabular}}
\label{tab:nci1_node}
\end{table*}

%% file: table/molhiv_idx.tex
\begin{table*}[t]
\small
\caption{The classification results (in \%) on the ogbg-molhiv dataset under edge density domain shift (source $\rightarrow$ target).  H0, H1, H2, and H3 denote the sub-datasets partitioned with edge density. \textbf{Bold} results indicate the best performance.}
\resizebox{1.0\textwidth}{!}{
\begin{tabular}{l|c|c|c|c|c|c|c|c|c|c|c|c}
\toprule
\textbf{Methods}
& H0$\rightarrow$H1
& H0$\rightarrow$H2
& H0$\rightarrow$H3
& H1$\rightarrow$H0
& H1$\rightarrow$H2
& H1$\rightarrow$H3
& H2$\rightarrow$H0
& H2$\rightarrow$H1
& H2$\rightarrow$H3
& H3$\rightarrow$H0
& H3$\rightarrow$H1
& H3$\rightarrow$H2 \\
\midrule
WL
& 45.4 & 44.9 & 54.4
& 54.9 & 53.7 & 56.1
& 46.7 & 46.0 & 51.1
& 52.2 & 45.3 & 50.8\\

GCN
& 57.5\scriptsize{$\pm$1.8} & 56.1\scriptsize{$\pm$2.1} & 54.1\scriptsize{$\pm$1.4}
& 58.0\scriptsize{$\pm$1.5} & 61.2\scriptsize{$\pm$1.6} & 54.9\scriptsize{$\pm$2.1}
& 54.3\scriptsize{$\pm$1.2} & 54.8\scriptsize{$\pm$1.9} & 65.4\scriptsize{$\pm$2.0}
& 50.9\scriptsize{$\pm$1.4} & 56.6\scriptsize{$\pm$1.4} & 58.1\scriptsize{$\pm$1.6}\\

GIN
& 58.1\scriptsize{$\pm$1.1} & 49.7\scriptsize{$\pm$1.5} & 55.5\scriptsize{$\pm$1.9}
& 55.3\scriptsize{$\pm$1.2} & 55.1\scriptsize{$\pm$1.3} & 56.6\scriptsize{$\pm$1.6}
& 55.2\scriptsize{$\pm$2.2} & 58.8\scriptsize{$\pm$1.2} & 56.8\scriptsize{$\pm$1.8}
& 50.5\scriptsize{$\pm$2.0} & 50.0\scriptsize{$\pm$1.1} & 56.2\scriptsize{$\pm$1.6} \\

GMT
& 50.7\scriptsize{$\pm$1.0} & 51.1\scriptsize{$\pm$2.1} & 59.8\scriptsize{$\pm$1.4}
& 57.5\scriptsize{$\pm$1.7} & 57.0\scriptsize{$\pm$1.8} & 54.0\scriptsize{$\pm$1.7}
& 58.8\scriptsize{$\pm$1.5} & 55.4\scriptsize{$\pm$1.0} & 58.1\scriptsize{$\pm$1.3}
& 52.0\scriptsize{$\pm$1.4} & 54.9\scriptsize{$\pm$1.0} & 52.0\scriptsize{$\pm$1.3} \\

CIN
& 51.8\scriptsize{$\pm$1.4} & 53.6\scriptsize{$\pm$1.4} & 56.9\scriptsize{$\pm$1.7}
& 54.1\scriptsize{$\pm$1.7} & 54.5\scriptsize{$\pm$1.7} & 54.2\scriptsize{$\pm$1.2}
& 57.3\scriptsize{$\pm$1.6} & 63.6\scriptsize{$\pm$2.1} & 57.6\scriptsize{$\pm$1.3}
& 53.9\scriptsize{$\pm$1.3} & 56.9\scriptsize{$\pm$1.8} & 58.7\scriptsize{$\pm$1.2}\\
\bottomrule
\end{tabular}}
\label{tab:hiv_idx}
\end{table*}

%% file: table/molhiv_node.tex
\begin{table*}[t]
\small
\caption{The classification results (in \%) on the ogbg-molhiv dataset under node density domain shift (source $\rightarrow$ target).  H0, H1, H2, and H3 denote the sub-datasets partitioned with node density. \textbf{Bold} results indicate the best performance.}
\resizebox{1.0\textwidth}{!}{
\begin{tabular}{l|c|c|c|c|c|c|c|c|c|c|c|c}
\toprule
\textbf{Methods}
& H0$\rightarrow$H1
& H0$\rightarrow$H2
& H0$\rightarrow$H3
& H1$\rightarrow$H0
& H1$\rightarrow$H2
& H1$\rightarrow$H3
& H2$\rightarrow$H0
& H2$\rightarrow$H1
& H2$\rightarrow$H3
& H3$\rightarrow$H0
& H3$\rightarrow$H1
& H3$\rightarrow$H2 \\
\midrule
WL
& 49.6 & 54.0 & 48.2
& 51.8 & 50.6 & 38.1
& 54.7 & 51.9 & 51.8
& 48.9 & 51.0 & 50.7 \\

GCN
& 57.8\scriptsize{$\pm$1.9} & 56.5\scriptsize{$\pm$1.3} & 54.5\scriptsize{$\pm$1.7}
& 57.4\scriptsize{$\pm$2.1} & 61.8\scriptsize{$\pm$2.0} & 49.8\scriptsize{$\pm$1.4}
& 54.2\scriptsize{$\pm$1.2} & 55.2\scriptsize{$\pm$1.3} & 61.5\scriptsize{$\pm$1.1}
& 51.9\scriptsize{$\pm$1.7} & 56.6\scriptsize{$\pm$1.3} & 56.8\scriptsize{$\pm$1.7} \\

GIN
& 59.1\scriptsize{$\pm$1.6} & 50.0\scriptsize{$\pm$1.7} & 56.8\scriptsize{$\pm$1.6}
& 53.3\scriptsize{$\pm$1.6} & 55.2\scriptsize{$\pm$1.3} & 57.0\scriptsize{$\pm$1.8}
& 54.8\scriptsize{$\pm$1.8} & 58.7\scriptsize{$\pm$1.4} & 53.1\scriptsize{$\pm$1.3}
& 50.7\scriptsize{$\pm$1.4} & 49.6\scriptsize{$\pm$1.7} & 55.3\scriptsize{$\pm$1.6} \\

GMT
& 51.1\scriptsize{$\pm$1.6} & 50.6\scriptsize{$\pm$1.2} & 58.4\scriptsize{$\pm$1.4}
& 57.2\scriptsize{$\pm$1.7} & 56.1\scriptsize{$\pm$1.6} & 60.9\scriptsize{$\pm$1.7}
& 55.9\scriptsize{$\pm$1.9} & 57.6\scriptsize{$\pm$1.3} & 55.4\scriptsize{$\pm$1.6}
& 53.4\scriptsize{$\pm$1.7} & 54.5\scriptsize{$\pm$2.0} & 51.2\scriptsize{$\pm$1.8} \\

CIN
& 52.3\scriptsize{$\pm$1.6} & 54.0\scriptsize{$\pm$1.9} & 56.5\scriptsize{$\pm$1.4}
& 53.4\scriptsize{$\pm$1.5} & 55.8\scriptsize{$\pm$1.4} & 48.7\scriptsize{$\pm$1.6}
& 57.5\scriptsize{$\pm$1.8} & 63.6\scriptsize{$\pm$1.8} & 49.3\scriptsize{$\pm$1.2}
& 51.1\scriptsize{$\pm$1.9} & 56.8\scriptsize{$\pm$1.6} & 52.8\scriptsize{$\pm$1.3} \\
\bottomrule
\end{tabular}}
\label{tab:hiv_node}
\end{table*}